\documentclass[11pt]{article}

\usepackage[preprint]{acl}

\usepackage{times}
\usepackage{latexsym}

\usepackage[T1]{fontenc}

\usepackage[utf8]{inputenc}

\usepackage{microtype}

\usepackage{inconsolata}

\usepackage{graphicx}
\usepackage{subcaption}

\usepackage{amsmath,amssymb,amsthm}
\usepackage{algorithm}
\usepackage{algpseudocode}
\usepackage{cleveref}
\usepackage{booktabs}

\DeclareMathAlphabet\mathbfcal{OMS}{cmsy}{b}{n}
\newcommand{\ten}[1]{\mathbfcal{#1}}
\newcommand{\mat}[1]{\mathbf{#1}}

\theoremstyle{plain}
\newtheorem{theorem}{Theorem}
\newtheorem{proposition}{Proposition}
\newtheorem{lemma}{Lemma}

\theoremstyle{definition}
\newtheorem{assumption}{Assumption}

\theoremstyle{remark}
\newtheorem{remark}{Remark}

\theoremstyle{plain}
\newtheorem{definition}{Definition}

\crefname{assumption}{assumption}{assumptions}
\Crefname{assumption}{Assumption}{Assumptions}

\title{\textsc{t-Muon}: Tensorizing Muon via t-SVD for Large Language Model Pre-training}

\author{
  Dongyang Li, \quad
  Mingsong Yan, \quad
  Zheng Zhang, \quad
  Sui Tang \\
  University of California, Santa Barbara \\
  Santa Barbara, CA 93106, USA \\
  \texttt{\{dongyang\_li,mingsongyan,suitang\}@ucsb.edu} \\
  \texttt{zhengzhang@ece.ucsb.edu}
}

\begin{document}
\maketitle
\begin{abstract}
Muon has recently emerged as a strong optimizer for large language model pretraining by exploiting matrix structure through polar orthogonalization. However, it processes each parameter matrix independently, whereas recent tensor-aware extensions of Muon jointly update groups of eligible matrix parameters through matrix unfoldings. In contrast, we propose \textsc{t-Muon}, a t-SVD-based generalization that groups structurally related Transformer gradients into third-order tensors and defines the corresponding Muon-style update directly through the $\mat{M}$-transform. We show that the exact \textsc{t-Muon} update is a steepest-descent direction under the $\star_{\mat{M}}$-spectral norm and establish a convergence guarantee. Across GPT-2-style and LLaMA-style models from 51M to 1.3B parameters trained on FineWeb and FineWeb-Edu, head-wise \textsc{t-Muon} with the identity transform consistently achieves lower validation perplexity than AdamW, Muon, and variants using DFT, DCT, or DST transforms. These gains persist in additional experiments with longer context lengths and larger token budgets, although the preferred tensorization mode varies across settings.
\end{abstract}

\section{Introduction}
\label{sec:introduction}
Pretraining large language models (LLMs) on large-scale datasets requires substantial computational resources, making optimization efficiency an important component of practical training \citep{brown2020language,touvron2023llama}. AdamW is a standard baseline optimizer for LLM pretraining, yet its coordinate-wise adaptation does not exploit the matrix structure of individual parameter blocks \citep{kingma2015adam,loshchilov2019decoupled}. Muon offers a matrix-aware alternative for two-dimensional weight matrices by orthogonalizing the gradient matrices \citep{jordan2024muon}. Subsequent studies have demonstrated its competitive performance and training efficiency relative to AdamW across a range of LLM-pretraining settings \citep{liu2025muon,shah2025muon}.

Muon's application of the polar map to each matrix independently leaves open how parameter blocks should be organized and jointly orthogonalized. Recent methods investigate these choices from different perspectives. Group Muon studies attention-head grouping \citep{zhang2026groupmuon}, while MuonBP uses predominantly device-local orthogonalization to reduce communication overhead \citep{khaled2025muonbp}. Rather than partitioning individual matrices, a complementary line of work organizes related matrix-valued parameters into a higher-order tensor \citep{zhang2026teon} to improve optimization effectiveness and practical training efficiency. 

However, existing tensor-aware extensions of Muon apply matrix orthogonalization to tensor unfoldings \citep{zhang2026teon,bogachev2026tensorion} rather than constructing the update from a tensor singular value decomposition (t-SVD) . Moreover, existing work has focused either on naturally tensor-valued parameters \citep{bogachev2026tensorion,qing2026muonc} or on a limited set of constructed tensor groups \citep{zhang2026teon}. Other structurally related parameter groupings therefore remain unexplored from a unified t-SVD perspective.

In this work, we investigate a transform-based geometry for tensorized Transformer updates. We introduce \textsc{t-Muon}, a t-SVD-based extension of Muon that stacks structurally related matrix gradients into third-order tensors and constructs their updates using the $\mat{M}$-transform and t-SVD \citep{kilmer2011factorization,kernfeld2015tensor,kilmer2021tensor}. This formulation separates the choice of which gradients are grouped from the choice of how the resulting tensor slices interact, yielding a unified family of \textsc{t-Muon} variants.

\paragraph{Contributions.}
Our main contributions are as follows:
\begin{itemize}
    \item We introduce \textsc{t-Muon}, which defines Muon-style updates for tensorized gradients through $\mat{M}$-transform and t-SVD.

    \item We test multiple modes to tensorize Transformer gradients as third-order tensors, including Layer, query-key-value (QKV), Head, and Head-stack tensorizations.
    
    \item We show that the exact \textsc{t-Muon} update is a norm-constrained steepest-descent direction under the corresponding $\star_{\mat{M}}$-spectral norm. We further establish a nonasymptotic stationarity guarantee for an idealized stochastic tensor method with Nesterov-type momentum.

    \item We compare AdamW, Muon, and \textsc{t-Muon} variants across identity, DFT, DCT, and DST transforms in GPT-2-style and LLaMA-style language-model pretraining on FineWeb and FineWeb-Edu. We evaluate models from 51M to 1.3B parameters and further test the methods under longer context lengths and larger token budgets. We find that head-wise \textsc{t-Muon} with the identity transform generally performs best in the main comparisons.
 
\end{itemize}

\section{Background}
\label{sec:background}
\subsection{Muon Optimizer}
\label{sec:background-structured-optimizers}

By replacing matrix-valued momentum with an approximate polar factor, Muon \citep{jordan2024muon} is an optimizer that exploits the matrix structure of neural network parameters and has emerged as a competitive alternative to AdamW for language-model pretraining \citep{liu2025muon,shah2025muon}. Let $\mat{W}_t\in\mathbb{R}^{m\times n}$ be a matrix parameter with stochastic gradient $\mat{G}_t\in\mathbb{R}^{m\times n}$ and momentum $\mat B_t\in\mathbb{R}^{m\times n}$. The Muon update is
\begin{align}
    \mat B_t &= \mu\mat B_{t-1} + (1-\mu)\mat G_t \nonumber\\
    \mat O_t &=\operatorname{Orth}(\mat{B}_t)\label{eq:bg-muon-update}\\
    \mat{W}_{t+1} &= \mat{W}_{t} - \eta_t\cdot\sqrt{m/n}\cdot \mat O_t\nonumber,  
\end{align}
where $\operatorname{Orth}(\cdot)$ orthogonalizes the matrix input. Specifically, if $\mat{B}_t=\mat{U}_t\mat{\Sigma}_t\mat{V}_t^{\top}$ is the thin SVD, then $\operatorname{Orth}(\mat B_t) = \mat{U}_t \mat{V}_t^{\top}$.

Practical Muon implementations replace the exact polar map with a small number of Newton--Schulz or related polynomial iterations \citep{higham2008functions,amsel2025polar,grishina2025accelerating}. Subsequent work has extended this template in several directions, including more efficient or accurate polar approximations \citep{amsel2025polar,grishina2025accelerating}, statistical calibration and structural balancing around the polar step \citep{si2025adamuon,chang2026muoneq}, alternative spectral targets for the update \citep{pang2026htmuon,dolatabadi2026numuon}, and different choices of orthogonalization domain or parameter geometry \citep{zhang2026groupmuon,khaled2025muonbp,li2026intrinsic}. 

An important line of work extends Muon by changing the parameter domain over which orthogonalization is applied. TEON stacks gradients of the same projection type from adjacent Transformer layers into a third-order tensor \citep{zhang2026teon}, while Tensorion targets parameters with native tensor structure, such as convolutional kernels and patch-embedding weights \citep{bogachev2026tensorion}. Together, these methods show that the choice of orthogonalization domain can affect optimization. However, these formulations couple tensor components according to a chosen matricization. In this paper, we instead use an orthogonal \(\mat M\)-transform to explicitly mix grouped matrices, consider several ways to group Transformer gradients into tensors, and define the update directly using the \(\mat M\)-product and t-SVD.

Concurrent with our work, Muon-C \citep{qing2026muonc} extends Muon to convolutional kernels. It applies a spatial discrete Fourier transform (DFT) to the kernel and then independently polarizes the resulting frequency-wise channel-transfer matrices. This design shares the structure of \textsc{t-Muon}, which applies separate polar maps to transformed frontal slices. However, Muon-C is specialized to the Fourier geometry induced by the translation equivariance of convolution. In contrast, \textsc{t-Muon} uses the $\mat{M}$-product and t-SVD to explore a general transform-based tensor formulation, in which both the tensorization and the transform are design choices. The two methods also target different models. Muon-C optimizes the native kernels of convolutional vision models, whereas \textsc{t-Muon} constructs tensors from structurally related matrices of Transformer language models.

\subsection{Tensors and t-SVD}
\label{sec:background-tsvd}

We introduce the tensor notation used throughout the paper. For a third-order tensor $\ten{X}\in\mathbb{R}^{m\times n\times K}$, the scalar $x_{ijk}$ denotes its $(i,j,k)$-th entry, and the matrix $\mat X_k:=\ten{X}_{:,:,k}\in\mathbb{R}^{m\times n}$ denotes its $k$-th frontal slice. For equal-shaped tensors $\ten{X},\ten{Y}\in\mathbb{R}^{m\times n\times K}$, the Frobenius inner product and its norm are
\begin{align}
    \langle\ten{X},\ten{Y}\rangle
    &:=
    \sum_{i=1}^{m}
    \sum_{j=1}^{n}
    \sum_{k=1}^{K}
    x_{ijk}y_{ijk}, \nonumber
    \\
    \|\ten{X}\|_{\mathrm{F}}
    &:=
    \sqrt{\langle\ten{X},\ten{X}\rangle}. \nonumber
\end{align}
We follow the transform-based tensor algebra of \citet{kilmer2011factorization,kernfeld2015tensor,kilmer2021tensor}.

\begin{definition}[$\mat{M}$-Transform]
\label{def:m-transform}
Let $\ten{X}\in\mathbb{R}^{m\times n\times K}$ and let $\mat{M}\in\mathbb{R}^{K\times K}$ be orthogonal. The $\mat{M}$-transform of $\ten{X}$ is
\begin{equation}
    \widehat{\ten{X}}
    :=
    \ten{X}\times_3\mat{M},
    \label{eq:bg-m-transform}
\end{equation}
whose frontal slices satisfy
\begin{equation}
 \widehat{\mat{X}}_k
    =
    \sum_{l=1}^{K}
    M_{kl}\mat X_l,
    \qquad k \in[K],
    \label{eq:bg-m-transform-slices}
\end{equation}
where $[K]:=\{1,\ldots,K\}$. If $\mat{M}$ is orthogonal, the inverse transform is
\begin{equation}
    \ten{X}
    =
    \widehat{\ten{X}}\times_3\mat{M}^{\top}.
    \label{eq:bg-inverse-m-transform}
\end{equation}
\end{definition}

\begin{definition}[$\mat{M}$-Product]
\label{def:m-product}
Let $\ten{X}\in\mathbb{R}^{m\times r\times K}$ and $\ten{Y}\in\mathbb{R}^{r\times n\times K}$. Their $\mat{M}$-product,
\begin{equation}
    \ten{Z}
    :=
    \ten{X}\star_{\mat{M}}\ten{Y}
    \in\mathbb{R}^{m\times n\times K},
    \label{eq:bg-m-product}
\end{equation}
is defined in the transformed domain by
\begin{equation*}
   \widehat{\mat{Z}}_k
    :=
    \widehat{\mat{X}}_k
    \widehat{\mat{Y}}_k,
    \qquad k\in[K],
\end{equation*}
followed by the inverse transform $\ten{Z} =\widehat{\ten{Z}}\times_3\mat{M}^{\top}.$
\end{definition}

\begin{definition}[t-SVD under the $\mat{M}$-Product]
\label{def:tsvd}
Let $\ten{X}\in\mathbb{R}^{m\times n\times K}$. A t-SVD \citep{kilmer2011factorization,kernfeld2015tensor,kilmer2021tensor} of
$\ten{X}$ under the $\mat{M}$-product is a factorization
\begin{equation}
    \ten{X}
    =
    \ten{U}
    \star_{\mat{M}}
    \ten{S}
    \star_{\mat{M}}
    \ten{V}^{\top},
    \label{eq:bg-tsvd}
\end{equation} where $ \ten{U} \in\mathbb{R}^{m\times m\times K},\ten{S} \in\mathbb{R}^{m\times n\times K}$, $\ten{V} \in\mathbb{R}^{n\times n\times K}$, and $\ten V^\top$ is defined as transpose of frontal slices. In the transformed domain, for every $k\in[K]$, \begin{equation*}
    \widehat{\mat{X}}_k = \widehat{\mat{U}}_k \widehat{\mat{S}}_k \widehat{\mat{V}}_k^\top,
\end{equation*} is a full SVD, and hence $\ten X =  \widehat{\ten X}\times_3 \mat M^\top$.
\end{definition}

\begin{remark}[Unitary DFT Transforms]
For the unitary Discrete Fourier Transform, complex SVD is applied, and it requires the complex transform-domain adjoint. Details are provided in \Cref{app:dft-extension}.
\end{remark}

\begin{remark}[Full SVD]
In \Cref{def:tsvd}, $\widehat{\mat{X}}_k = \widehat{\mat{U}}_k \widehat{\mat{S}}_k \widehat{\mat{V}}_k^\top$ is the full SVD, and we define $\operatorname{Orth}(\widehat{\mat{X}}_k) = \left(\widehat{\mat{U}}_k\right)_{:,:r_k} \left(\widehat{\mat{V}}_k^\top\right)_{:r_k,:}$, where $r_k=\operatorname{rank}(\widehat{\mat{X}}_k)$. For simplicity, throughout the remainder of the paper, we write $\operatorname{Orth}(\widehat{\mat{X}}_k) = \widehat{\mat{U}}_k\widehat{\mat{V}}_k^\top$.
\end{remark}

\subsection{Norm-Constrained Steepest Descent}
\label{sec:background-steepest-descent}

Muon can be interpreted as norm-constrained steepest descent \citep{bernstein2024oldoptimizer}. Let $\mat{W}_t\in\mathbb{R}^{m\times n}$ be a matrix parameter and let $\mat{G}_t=\nabla_{\mat{W}}\mathcal{L}(\mat{W}_t)$ be its gradient. For a step size $\eta_t>0$, the spectral-norm-constrained update solves
\begin{equation}
    \Delta\mat{W}_t^{\star}
    \in
    \operatorname*{arg\,min}_{
        \|\Delta\mat{W}\|_{\mathrm{spec}}\leq\eta_t
    }
    \left\langle
        \mat{G}_t,
        \Delta\mat{W}
    \right\rangle .
    \label{eq:bg-muon-steepest-descent}
\end{equation}

By spectral--nuclear norm duality, the optimization problem in \Cref{eq:bg-muon-steepest-descent} can be solved by the matrix polar factor. The same framework extends to tensor-valued parameters by choosing an appropriate norm on the tensor space. TEON uses the matrix spectral norm of a selected tensor unfolding \citep{zhang2026teon}, whereas Tensorion introduces a relaxed tensor norm whose dual is defined through the nuclear norms of candidate unfoldings \citep{bogachev2026tensorion}. In \Cref{sec:steepest-descent}, we instead define a transform-domain tensor spectral norm that yields the exact \textsc{t-Muon} update.

\section{\textsc{t-Muon}}
\label{sec:methodology}

\subsection{\textsc{t-Muon} Updates}
\label{sec:tsvd-muon-update}

Now, we introduce \textsc{t-Muon}, a tensor-aware optimizer that orthogonalizes tensor parameters by utilizing the M-transform and t-SVD. The algorithm is presented in \Cref{alg:tsvd-muon}.

Given a tensorized or tensor parameter $\ten{W}_t\in \mathbb{R}^{m \times n \times K}$ and its momentum direction $\ten{B}_t  \in \mathbb{R}^{m \times n \times K}$, \textsc{t-Muon} first applies the M-transform with a pre-selected real orthogonal $\mat{M}\in\mathbb{R}^{K\times K}$, $\widehat{\ten{B}_t} = \ten{B}_t\times_3\mat{M}.$ 

Then, \textsc{t-Muon} applies the orthogonalization to each transformed slice $(\widehat{\mat{B}_t})_k$,  $$(\widehat{ \mat{D}_t})_k=\operatorname{Orth}\left((\widehat{\mat{B}_t})_k\right).$$

Finally, an inverse M-transform is applied, $ \ten D_t = \widehat{\ten D_t}\times_3 \mat{M}^\top,$ and the actual update is $$\ten W_{t+1} = \ten W_t - \eta\cdot \ten D_t.$$ 

If $\ten {W}_t$ is organized from a matrix-valued parameter $\mat W$, the update of $\mat W$ is multiplied by an extra scaling factor $\sqrt\frac{\mathrm{fan}_{\mathrm{out}}(\mat W)} {\mathrm{fan}_{\mathrm{in}}(\mat W)}$. It is determined by the original shape of the matrix parameter and therefore need not coincide with the tensor dimensions \(m\) and \(n\). 

When \(\mat M\) is the DFT, the inverse transform is real in exact arithmetic for real input tensors. In floating-point arithmetic, however, it may contain tiny imaginary roundoff errors, therefore we take the real part to enforce a real-valued update. For real orthogonal transforms, this step has no effect.

Depending on the tensorization, the identity-transform case may coincide with either full-matrix or blockwise Muon. Compared to TEON and Tensorion \citep{zhang2026teon,bogachev2026tensorion} that apply matrix orthogonalization to tensor unfolding, \textsc{t-Muon} applies separate polar maps to the transformed frontal slices and therefore induces a distinct update geometry.

\begin{algorithm}[t]
\caption{\textsc{t-Muon} for One Tensorized Parameter Group}
\label{alg:tsvd-muon}
\begin{algorithmic}[1]
\Require Initial tensorized or tensor parameter $\ten{W}_{0}\in\mathbb{R}^{m\times n\times K}$
\Require Mini-batch losses $\{\mathcal{L}_{t}\}_{t=0}^{T_{\mathrm{iter}}-1}$, learning-rate schedule $\{\eta_t\}_{t=0}^{T_{\mathrm{iter}}-1}$, weight decay $\lambda\geq 0$, and momentum $\mu\in[0,1)$
\Require Real orthogonal transform $\mat{M}\in\mathbb{R}^{K\times K}$ or unitary DFT transform $\mat{M}\in\mathbb{C}^{K\times K}$, and number of polar iterations $T_{\mathrm{pol}}$
\Ensure Updated tensorized parameter  $\ten{W}_{T_{\mathrm{iter}}}$

\State $\ten{B}_{-1}\gets \mathbf{0}$

\For{$t=0,\ldots,T_{\mathrm{iter}}-1$}

    \State
    $\ten{G}_{t}
    \gets
    \nabla_{\ten{W}}\mathcal{L}_{t}(\ten{W}_{t})$

    \State
    $\ten{B}_{t}
    \gets
    \mu\ten{B}_{t-1}+\ten{G}_{t}$

    \State
    $\ten{Y}_{t}
    \gets
    \ten{G}_{t}+\mu\ten{B}_{t}$

    \State
    $\widehat{\ten{Y}}_{t}
    \gets
    \ten{Y}_{t}\times_{3}\mat{M}$

    \State
    $\widehat{\ten{D}}_{t}\gets \mathbf{0}$

    \For{$k=1,\ldots,K$}

        \State
        $\left(\widehat{\mat{D}}_{t}\right)_k
        \gets
        \operatorname{Orth}_{T_{\mathrm{pol}}}
        \left(
        \left(\widehat{\mat{Y}}_{t}\right)_k
        \right)$

    \EndFor

    \State
    $\ten{D}_{t}
    \gets
    \operatorname{Re}
    \left(
    \widehat{\ten{D}}_{t}
    \times_{3}
    \mat{M}^{*}
    \right)$

    \State
    $\ten{W}_{t+1}
    \gets
    (1-\eta_t\lambda)\ten{W}_{t}
    -
    \eta_t\ten{D}_{t}$

\EndFor

\State \Return $\ten{W}_{T_{\mathrm{iter}}}$

\end{algorithmic}
\end{algorithm}

\subsection{Tensorization Mode}
\label{sec:tensorization}

As discussed in \Cref{sec:background-structured-optimizers}, existing structured extensions of Muon redefine the parameter domain through blockwise partitioning or higher-order tensorization. We focus on tensorization and consider four main Transformer-specific tensorization modes.

We form a third-order tensor by stacking equal-shaped, structurally related matrix parameters as frontal slices. The third mode indexes adjacent layers, query--key--value projections, or attention heads. Let $\mat{G}_{i}^{(\ell)} \in \mathbb{R}^{m \times n}$ denote the gradient matrix associated with projection $i \in \{Q,K,V\}$ in layer $\ell$. 

\begin{itemize}
    \item \textbf{Layer mode.}
Following TEON \citep{zhang2026teon}, we stack the same projection
from two adjacent Transformer layers as
$\ten{G}_{i,\mathrm{layer}}^{(\ell-1,\ell)}
\in \mathbb{R}^{m \times n \times 2}$ with
\[
\begin{aligned}
    \left(\mat{G}_{i,\mathrm{layer}}^{(\ell-1,\ell)}\right)_1
    &= \mat{G}_{i}^{(\ell-1)}, \\
    \left(\mat{G}_{i,\mathrm{layer}}^{(\ell-1,\ell)}\right)_2
    &= \mat{G}_{i}^{(\ell)}.
\end{aligned}
\]
    \item \textbf{QKV mode.}
    We stack the query, key, and value projection gradients from the same Transformer layer as $\ten{G}_{\mathrm{QKV}}^{(\ell)} \in \mathbb{R}^{d \times d \times 3}$: $$\begin{aligned}
        \left(
        \mat{G}_{\mathrm{QKV}}^{(\ell)}
        \right)_1
        &= \mat{G}_{Q}^{(\ell)}, \\
        \left(
        \mat{G}_{\mathrm{QKV}}^{(\ell)}
        \right)_2
        &= \mat{G}_{K}^{(\ell)}, \\
        \left(
        \mat{G}_{\mathrm{QKV}}^{(\ell)}
        \right)_3
        &= \mat{G}_{V}^{(\ell)} .
    \end{aligned}$$
    \item \textbf{Head mode.} 
    We partition an attention-projection gradient matrix into head-specific row blocks and stack these blocks along the third mode as $\ten{G}_{i,\mathrm{head}}^{(\ell)} \in \mathbb{R}^{d_H \times d \times H}$ with
        $$\begin{aligned}
        \left(
        \mat{G}_{i,\mathrm{head}}^{(\ell)}
        \right)_h
        &= \mat{G}_{i,h}^{(\ell)},
        \qquad h \in [H],
    \end{aligned}$$where $\mat{G}_{i,h}^{(\ell)}
        \in \mathbb{R}^{d_H \times d}$ and $d_H = d/H$. Under the identity transform, this mode orthogonalizes each head block separately as studied by \citet{zhang2026groupmuon}. Our contribution is the transform-based tensor formulation, analysis, and evaluation—not head-wise orthogonalization itself.
    \item \textbf{Head-stack mode.} We stack the head blocks of the query, key, and value projections in the same layer as $\ten{G}_{\mathrm{head}}^{(\ell)}\in\mathbb{R}^{d_H \times d \times 3H}$. 
\end{itemize}

\subsection{Steepest-Descent under $\star_{\mat{M}}$-Spectral Norm }
\label{sec:steepest-descent}

Now, we show that the exact \textsc{t-Muon} update follows the steepest-descent framework. First, we define the $\star_{\mat{M}}$-spectral norm, which is the largest matrix spectral norm among the transformed frontal slices.
\begin{definition}[$\star_{\mat{M}}$-Spectral Norm]
\label{def:star-m-spectral-norm}
Let $\ten{X}\in\mathbb{R}^{m\times n\times K}$ and let $\mat{M}\in\mathbb{R}^{K\times K}$ be an orthogonal transform. Define the transformed tensor
\begin{equation*}
    \widehat{\ten{X}}
    :=
    \ten{X}\times_{3}\mat{M}.
\end{equation*}
The $\star_{\mat{M}}$-spectral norm of $\ten{X}$ is defined as
\begin{equation}
    \|\ten{X}\|_{\star_{\mat{M}}} := \max_{k\in[K]} \left\| \widehat{\mat{X}}_k \right\|_{\mathrm{spec}},
    \label{eq:star-m-spectral-norm}
\end{equation}
where $\|\cdot\|_{\mathrm{spec}}$ denotes the matrix spectral norm.
\end{definition} 

For a step size $\eta>0$, the steepest descent under $\|\cdot\|_{\star_{\mat{M}}}$ therefore selects
\begin{equation}
    \Delta\ten{W}^{\star}
    \in
    \operatorname*{arg\,min}_{
        \|\Delta\ten{W}\|_{\star_{\mat{M}}}\leq \eta
    }
    \left\langle
        \ten{G},
        \Delta\ten{W}
    \right\rangle .
    \label{eq:star-m-steepest-descent}
\end{equation}

The solution to \Cref{eq:star-m-steepest-descent} is obtained by applying the matrix polar map independently to the transformed frontal slices.

\begin{definition}[$\mat{M}$-Polar Map]
\label{def:m-polar-map}
Let $\mat{M}\in\mathbb{R}^{K\times K}$ be orthogonal, and let $\ten{G}\in\mathbb{R}^{m\times n\times K}$. Define its $\mat{M}$-transform as \(\widehat{\ten{G}}:= \ten{G}\times_3\mat{M}\). For each $k\in[K]$, let
\begin{equation*}
    \widehat{\mat{G}}_k
    =
    \mat{U}_k
    \mat{\Sigma}_k
    \mat{V}_k^{\top}
\end{equation*}
be the corresponding thin matrix SVD. Define $\widehat{\ten{O}}\in\mathbb{R}^{m\times n\times K}$ by
\begin{equation}
    \widehat{\mat{O}}_k
    :=
    \operatorname{Orth}
    \left(
        \widehat{\mat{G}}_k
    \right) =
    \mat{U}_k\mat{V}_k^{\top},
    \qquad k\in[K].
    \label{eq:m-polar-transformed-slices}
\end{equation}
The $\mat{M}$-polar map of $\ten{G}$ is then defined as
\begin{equation}
    \operatorname{Orth}_{\mat{M}}(\ten{G})
    :=
    \widehat{\ten{O}}\times_3\mat{M}^{\top}.
    \label{eq:m-polar-map}
\end{equation}
\end{definition}

\begin{proposition}[\textsc{t-Muon} as Steepest Descent]
\label{prop:star-m-steepest-descent}
Let $\mat{M}\in\mathbb{R}^{K\times K}$ be orthogonal and let $\ten{G}\in\mathbb{R}^{m\times n\times K}$. Then, a solution to the $\star_{\mat{M}}$-spectral-norm-constrained steepest-descent problem in \Cref{eq:star-m-steepest-descent} is
\begin{equation}
    \Delta\ten{W}^{\star}
    =
    -\eta\cdot \operatorname{Orth}_{\mat{M}}(\ten{G}).
    \label{eq:star-m-direction}
\end{equation}
\end{proposition}

The complete proof is provided in \Cref{app:real-steepest-descent-proof}. Since $\mat{M}$ is orthogonal, the mode-three transform preserves the Frobenius inner product, and the optimization problem \Cref{eq:star-m-steepest-descent} decomposes independently across the transformed frontal slices. Each transformed slice is then constrained only in its spectral norm. Therefore, spectral--nuclear norm duality identifies its polar factor as an optimal solution. Applying the inverse transform then yields \Cref{eq:star-m-direction}, establishing the exact \textsc{t-Muon} update as steepest descent under the $\star_{\mat{M}}$-spectral norm. The derivation above assumes a real orthogonal transform, while the unitary DFT case is treated separately in \Cref{app:dft-steepest-descent}.

\subsection{Convergence Guarantee}
\label{sec:convergence}

Our analysis follows the convergence analysis of MuonBP \citep{khaled2025muonbp}, which builds on \citet{kovalev2025understanding} and \citet{li2025note}. Let $f:\mathbb{R}^{m\times n\times K}\rightarrow\mathbb{R}$ be a differentiable objective defined on a tensor parameter $\ten{W}\in\mathbb{R}^{m\times n\times K}$. We equip this space with the Frobenius inner product, an arbitrary norm $\|\cdot\|$, and its dual norm $\|\cdot\|_{*}$. We further assume that $f$ is bounded below by $f^{\star}$.

The analysis relies on smoothness, an unbiased stochastic-gradient estimator with bounded variance, and equivalence between the dual and Frobenius norms.
\begin{assumption}[Smoothness]
\label{ass:smoothness}
The objective $f$ is $\beta$-smooth with respect to $\|\cdot\|$; that is, for all $\ten{W},\ten{W}'\in\mathbb{R}^{m\times n\times K}$,
\begin{equation}
    \left\|
        \nabla f(\ten{W})-\nabla f(\ten{W}')
    \right\|_{*}
    \leq
    \beta
    \left\|
        \ten{W}-\ten{W}'
    \right\|.
    \label{eq:convergence-smoothness}
\end{equation}
\end{assumption}

\begin{assumption}[Stochastic Gradient]
\label{ass:stochastic-gradient}
Let $\ten{G}(\ten{W};\zeta)$ denote the stochastic-gradient estimator, where $\zeta\sim\mathcal{D}_{\mathrm{data}}$ represents the sampled mini-batch. We assume that the estimator is unbiased and has bounded Frobenius variance:
\begin{equation}
\begin{aligned}
    &\mathbb{E}_{\zeta}
    \left[
        \ten{G}(\ten{W};\zeta)
    \right]
    =
    \nabla f(\ten{W}), \\
    &\mathbb{E}_{\zeta}
    \left[
        \left\|
            \ten{G}(\ten{W};\zeta)
            -
            \nabla f(\ten{W})
        \right\|_{\mathrm{F}}^{2}
    \right]
    \leq
    \sigma^{2}.
\end{aligned}
\label{eq:convergence-stochastic-gradient}
\end{equation}
\end{assumption}

\begin{assumption}[Norm Equivalence]
\label{ass:norm-equivalence}
For the selected norm $\|\cdot\|$, there exists a constant $\rho>0$ such that
\begin{equation}
    \|\ten{A}\|_{*} \leq \rho\|\ten{A}\|_{\mathrm{F}},
    \label{eq:convergence-norm-equivalence}
\end{equation}
for every $\ten{A}\in\mathbb{R}^{m\times n\times K}$.
\end{assumption}

At each iteration, the method applies unnormalized Nesterov-type momentum before solving a linear optimization problem over the unit ball of the selected norm. Let $\ten{G}_{t}=\ten{G}(\ten{W}_{t};\zeta_{t})$ and initialize $\ten{B}_{-1}=\mathbf{0}$. For $\mu\in[0,1)$, we compute
\begin{equation}
\begin{aligned}
    \ten{B}_{t}
    &=
    \mu\ten{B}_{t-1}+\ten{G}_{t}, \\
    \ten{Y}_{t}
    &=
    \ten{G}_{t}+\mu\ten{B}_{t},
\end{aligned}
\label{eq:convergence-momentum}
\end{equation}
and select
\begin{equation}
\begin{aligned}
    \ten{D}_{t}
    &\in
    \operatorname*{arg\,max}_{\|\ten{D}\|\leq 1}
    \left\langle
        \ten{Y}_{t},
        \ten{D}
    \right\rangle, \\
    \ten{W}_{t+1}
    &=
    \ten{W}_{t}-\eta \ten{D}_{t},
\end{aligned}
\label{eq:convergence-update}
\end{equation}
where $\eta>0$ is a constant step size.

\begin{theorem}[Tensor Norm-Constrained Optimization with Unnormalized Nesterov Momentum]
\label{thm:convergence}
Suppose that \Cref{ass:smoothness,ass:stochastic-gradient,ass:norm-equivalence} hold and that $f$ is bounded below by $f^{\star}$. Define
\begin{equation}
    \Delta_{0}
    :=
    f(\ten{W}_{0})-f^{\star}.
\end{equation}
Then, for every $\eta>0$, $\mu\in[0,1)$, and $T_{\mathrm{iter}}\geq 1$, the iterates generated by \Cref{eq:convergence-update} satisfy
\begin{equation}
\begin{aligned}
&\mathbb{E}\!\left[
    \min_{0 \leq t < T_{\mathrm{iter}}}
    \left\|
        \nabla f(\ten{W}_t)
    \right\|_{*}
\right]
\\
&\quad\leq
\frac{\Delta_0}
     {\eta T_{\mathrm{iter}}}
+
\frac{
    2\mu^2\sqrt{2\beta\Delta_0}
}{
    (1-\mu)T_{\mathrm{iter}}
}
\\
&\qquad
+
\frac{2\beta\eta\mu^2}
     {1-\mu}
+
2\rho\sigma\chi_\mu
+
\frac{\beta\eta}{2}, 
\end{aligned}
\label{eq:convergence-bound}
\end{equation}
where $$ \chi_{\mu}
    :=
    (1-\mu)
    \sqrt{
        2\mu
        +
        \frac{1}{1-\mu^{2}}
    }.$$
\end{theorem}

The theorem bounds the smallest dual-gradient norm encountered during optimization and separates the effects of the initial optimality gap, momentum, stochastic-gradient noise, and the step size. When $\|\cdot\|=\|\cdot\|_{\star_{\mat{M}}}$, the exact linear oracle in \Cref{eq:convergence-update} is the exact \textsc{t-Muon} direction derived in \Cref{sec:steepest-descent}. The guarantee therefore applies to the idealized exact polar update; the complete proof is provided in the \Cref{app:convergence-proof}.

\section{Experiments}
\label{sec:experiments}
We evaluate \textsc{t-Muon} on decoder-only language-model pretraining. We first compare AdamW, Muon, and \textsc{t-Muon} variants on GPT-2-style and LLaMA-style models \citep{radford2019language,touvron2023llama} trained on 1B FineWeb tokens \citep{penedo2024fineweb}. Then, we test whether these comparisons hold under longer context lengths and larger token budgets. Finally, we scale up to LLaMA-style models with up to 1.3B parameters trained on up to 13.1B tokens of FineWeb-Edu \citep{penedo2024fineweb}.

\subsection{Experimental Setup}
\label{sec:main-pretraining-setup}

Our main evaluation covers seven decoder-only models: four GPT-2-style models and three LLaMA-style models, with parameter counts ranging from approximately 50 million to 780 million. The FineWeb-Edu experiments add a LLaMA-1B model with 1.34B parameters. The model architectures are summarized in \Cref{tab:model-configurations}. Detailed hyperparameters are provided in \Cref{tab:main-hyperparameters}.

We use the same optimizer assignment across all models. For the Muon baseline, eligible two-dimensional attention and MLP weights receive Muon updates, while embeddings, the language-model head, and other non-matrix parameters receive AdamW updates. For the \textsc{t-Muon} family, the query, key, and value projections receive tensorized updates; the attention-output and MLP matrices receive standard Muon updates; and the remaining parameters receive AdamW updates. We evaluate the Layer, QKV, Head, and Head-stack tensorization modes under the identity (-$I$), DFT (-$F$), and orthonormal type-II DCT (-$C$) and DST (-$S$) transforms \citep{kernfeld2015tensor,kilmer2021tensor,martucci1994symmetric}; for the two-slice Layer mode, the DCT and DST coincide with the DFT.

Unless otherwise stated, we sweep each optimizer over the learning-rate grid in \Cref{tab:lr-grids} with three random seeds. We select the learning rate that achieves the lowest mean final validation perplexity across seeds and report the corresponding mean and standard deviation.

\subsection{Main Comparisons on FineWeb}
\label{sec:fineweb-main}

\paragraph{GPT-2-style models.}
\label{sec:gpt-pretraining}
We evaluate \textsc{t-Muon} on four GPT-2-style models: GPT-50M, GPT-Small, GPT-Base, and GPT-Large. The 1B FineWeb tokens are tokenized by the GPT-2 tokenizer with vocabulary size 50,257 \citep{radford2019language}. \Cref{tab:gpt-ppl-main} compares AdamW and Muon with the two strongest \textsc{t-Muon} variants, while the complete comparison is reported in \Cref{tab:gpt-ppl-full}.

\begin{table}[t]
    \centering
    \setlength{\tabcolsep}{3pt}
    \begin{tabular}{@{}lcccc@{}}
        \toprule
        \textbf{Optimizer}
        & \textbf{50M}
        & \textbf{Small}
        & \textbf{Base}
        & \textbf{Large} \\
        \midrule
        AdamW
        & 47.353
        & 42.623
        & 38.641
        & 38.058 \\
        Muon
        & 42.235
        & 35.483
        & 30.350
        & 27.953 \\
        \midrule
        \multicolumn{5}{l}{\textit{\textsc{t-Muon}}} \\
        Head-$I$
        & 40.568
        & 33.893
        & 29.074
        & \textbf{26.959} \\
        Head-stack-$I$
        & \textbf{40.515}
        & \textbf{33.767}
        & \textbf{29.071}
        & 26.966 \\
        \bottomrule
    \end{tabular}
    \caption{
        Mean validation perplexity on GPT-2-style models trained on 1B-FineWeb tokens. Lower is better, and the best result in each column is bold.
    }
    \label{tab:gpt-ppl-main}
\end{table}

The GPT-2 results show that Head-$I$ achieves the lowest validation perplexity at every evaluated scale, outperforming both AdamW and Muon. Under the identity transform, Head-$I$ and Head-stack-$I$ implement the same update for the GPT-2 parameterization, which is consistent with their nearly identical validation perplexities. The complete comparison in \Cref{tab:gpt-ppl-full} further shows that, for every tensorization mode, the identity transform outperforms the other transforms. \Cref{fig:gpt-train} compares the training loss trajectories of Muon and Head-$I$ at their selected learning rates across all four model scales. Head-$I$ also reaches Muon's final smoothed training loss in fewer optimization steps, with step speedups of $1.14$--$1.19\times$; 

\paragraph{LLaMA-style models.}
\label{sec:llama-pretraining}
To examine whether the GPT-2 findings extend to a different Transformer architecture, we repeat the same experiment on LLaMA-60M, LLaMA-130M, and LLaMA-350M models. The 1B FineWeb tokens are tokenized by the LLaMA-2 tokenizer with vocabulary size 32,000 \citep{touvron2023llama2}. \Cref{tab:llama-ppl-main} presents the baselines and the two strongest \textsc{t-Muon} variants, with the complete results provided in \Cref{tab:llama-ppl-full}.

\begin{table}[t]
    \centering
    \setlength{\tabcolsep}{5pt}
    \begin{tabular}{@{}lccc@{}}
        \toprule
        \textbf{Optimizer}
        & \textbf{60M}
        & \textbf{130M}
        & \textbf{350M} \\
        \midrule
        AdamW
        & 32.166
        & 28.151
        & 24.029 \\
        Muon
        & 26.558
        & 22.363
        & 19.386 \\
        \midrule
        \multicolumn{4}{@{}l}{\textit{\textsc{t-Muon}}} \\
        Head-$I$
        & 25.948
        & \textbf{21.915}
        & \textbf{19.055} \\
        Head-stack-$I$
        & \textbf{25.930}
        & \textbf{21.915}
        & 19.058 \\
        \bottomrule
    \end{tabular}
    \caption{
        Mean validation perplexity on LLaMA-style models trained on 1B-FineWeb tokens. Lower is better, and the best result in each column is bold.
    }
    \label{tab:llama-ppl-main}
\end{table}

The LLaMA experiments reinforce the main GPT-2 finding: Head-$I$ outperforms AdamW and Muon at every evaluated scale. As in the GPT-2 implementation, Head-$I$ and Head-stack-$I$ reduce to the same head-wise Muon update, which is consistent with their nearly identical validation perplexities. Also, under the identity transform, Layer-$I$ independently applies Muon to each layer slice, while QKV-$I$ applies Muon to each attention matrix separately; both therefore reduce to standard Muon. The complete results in \Cref{tab:llama-ppl-full} also show that none of the other transforms improves over its identity-transform counterpart. \Cref{fig:llama-train} presents the corresponding training-loss trajectories for Muon and Head-$I$ at their selected learning rates across the three LLaMA-style models.

\subsection{Longer Contexts and Token Budgets}
\label{app:robustness}

Next, we examine whether the main comparison extends beyond the shared 1,024-token context length and 1B-token training budget. The compact tables in each part report AdamW, Muon, and the strongest \textsc{t-Muon} variant for each setting; the complete comparisons are provided in \Cref{tab:long-context-full,tab:long-budget-full}.

\paragraph{Longer Contexts.}
We increase the context length to 2,048 for GPT-Small and to 4,096 for GPT-Base and LLaMA-350M. These experiments use seed 42 only, and the learning rate is selected separately for each optimizer according to final validation perplexity. As shown in \Cref{tab:long-context-main}, the best \textsc{t-Muon} variant outperforms both AdamW and Muon in each longer-context setting. The best tensorization mode varies across the three models: Head-$I$ performs best on GPT-Small, Layer-$F$ performs best on GPT-Base, and Head-stack-$I$ performs best on LLaMA-350M. Notably, GPT-Base is the first evaluated setting in which a DFT-transform variant performs best: Layer-$F$ slightly outperforms Layer-$I$ and achieves the lowest validation perplexity overall.

\begin{table}[t]
    \centering
    \setlength{\tabcolsep}{2.5pt}
    \begin{tabular}{@{}lcccc@{}}
        \toprule
        \textbf{Setting}
        & \textbf{Ctx.}
        & \textbf{AdamW}
        & \textbf{Muon}
        & \textbf{\textsc{t-Muon}} \\
        \midrule
        GPT-Small
        & 2048
        & 65.99
        & 41.87
        & \shortstack{\textbf{36.91}\\[-1pt]
        {\scriptsize Head-$I$}}\\
        GPT-Base
        & 4096
        & 221.50
        & 55.90
        & \shortstack{\textbf{49.77}\\[-1pt]
        {\scriptsize Layer-$F$}} \\
        LLaMA-350M
        & 4096
        & 70.25
        & 28.71
        & \shortstack{\textbf{27.31}\\[-1pt]
        {\scriptsize Head-stack-$I$}} \\
        \bottomrule
    \end{tabular}
    \caption{
        Single-seed validation perplexity under longer contexts. The strongest \textsc{t-Muon} variant for each setting is identified below its score. Ctx.\ denotes context length in tokens.
    }
    \label{tab:long-context-main}
\end{table}

\paragraph{Longer Token Budgets.}
We extend training to 3B tokens for GPT-Small, 7B tokens for GPT-Base, and 2.2B tokens for LLaMA-130M. Learning rates are selected by the three-seed sweep. We report the mean final validation perplexity over all three seeds in \Cref{tab:long-budget-main}. The larger-budget experiments again favor \textsc{t-Muon} over AdamW and Muon in all three settings. Each selected method uses the identity transform, but the strongest tensorization changes from QKV on GPT-Small to layer mode on GPT-Base and head mode on LLaMA-130M.

\begin{table}[t]
    \centering
    \setlength{\tabcolsep}{2.5pt}
    \begin{tabular}{@{}lcccc@{}}
        \toprule
        \textbf{Setting}
        & \textbf{Tokens}
        & \textbf{AdamW}
        & \textbf{Muon}
        & \textbf{\textsc{t-Muon}} \\
        \midrule
        GPT-Small
        & 3B
        & 28.97
        & 27.94
        & \shortstack{\textbf{27.03}\\[-1pt]
          {\scriptsize QKV-$I$}} \\
        GPT-Base
        & 7B
        & 20.96
        & 20.47
        & \shortstack{\textbf{19.85}\\[-1pt]
          {\scriptsize Layer-$I$}} \\
        LLaMA-130M
        & 2.2B
        & 21.51
        & 19.35
        & \shortstack{\textbf{19.00}\\[-1pt]
          {\scriptsize Head-$I$}} \\
        \bottomrule
    \end{tabular}
    \caption{
        Mean validation perplexity under larger token budgets for FineWeb datasets. The strongest \textsc{t-Muon} variant for each setting is identified below its score.}
    \label{tab:long-budget-main}
\end{table}

\begin{figure*}[t]
    \centering
    \begin{subfigure}[t]{0.49\textwidth}
        \centering
        \includegraphics[width=\linewidth]
        {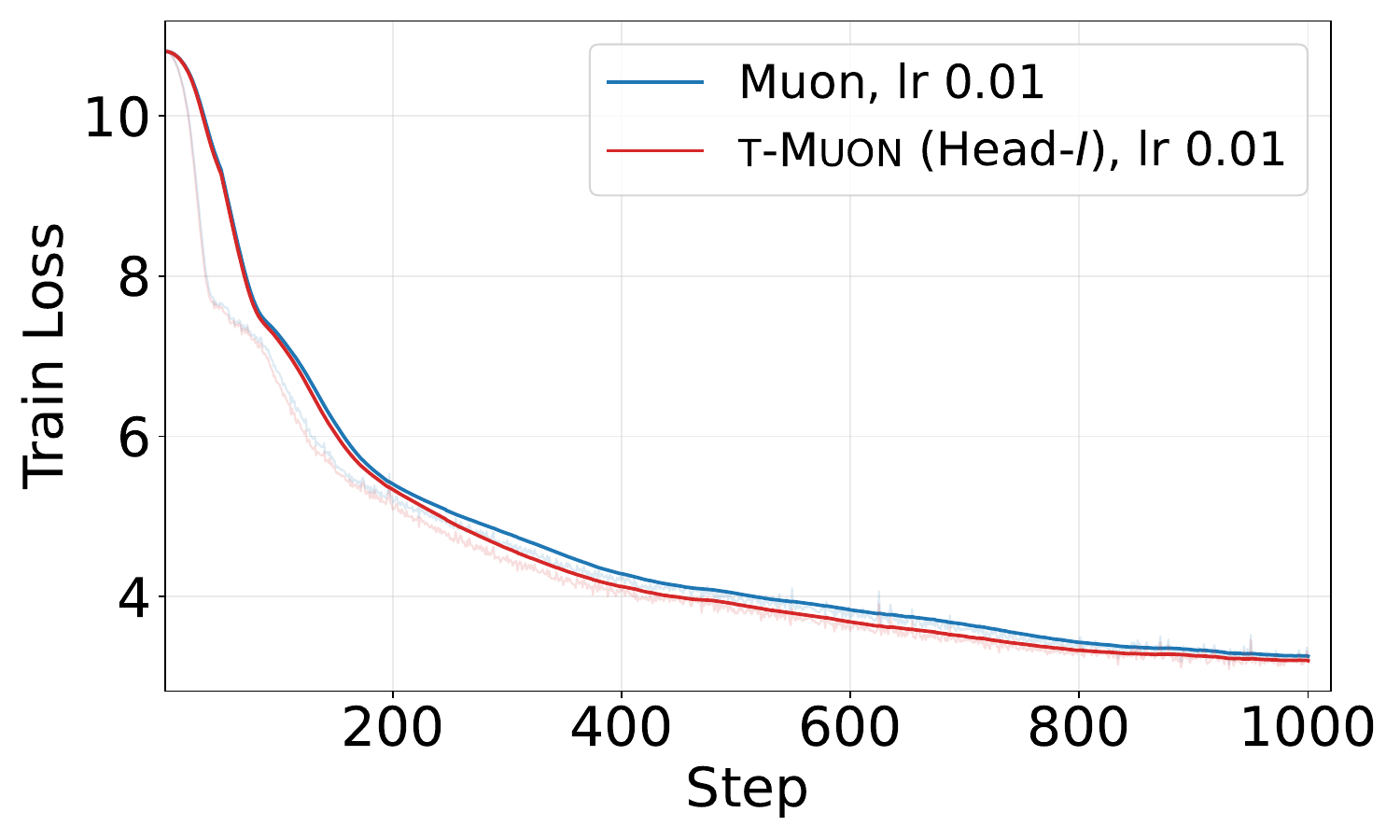}
        \caption{LLaMA-1B (13.1B tokens), early training}
        \label{fig:edu-1b-early}
    \end{subfigure}\hfill%
    \begin{subfigure}[t]{0.49\textwidth}
        \centering
        \includegraphics[width=\linewidth]
        {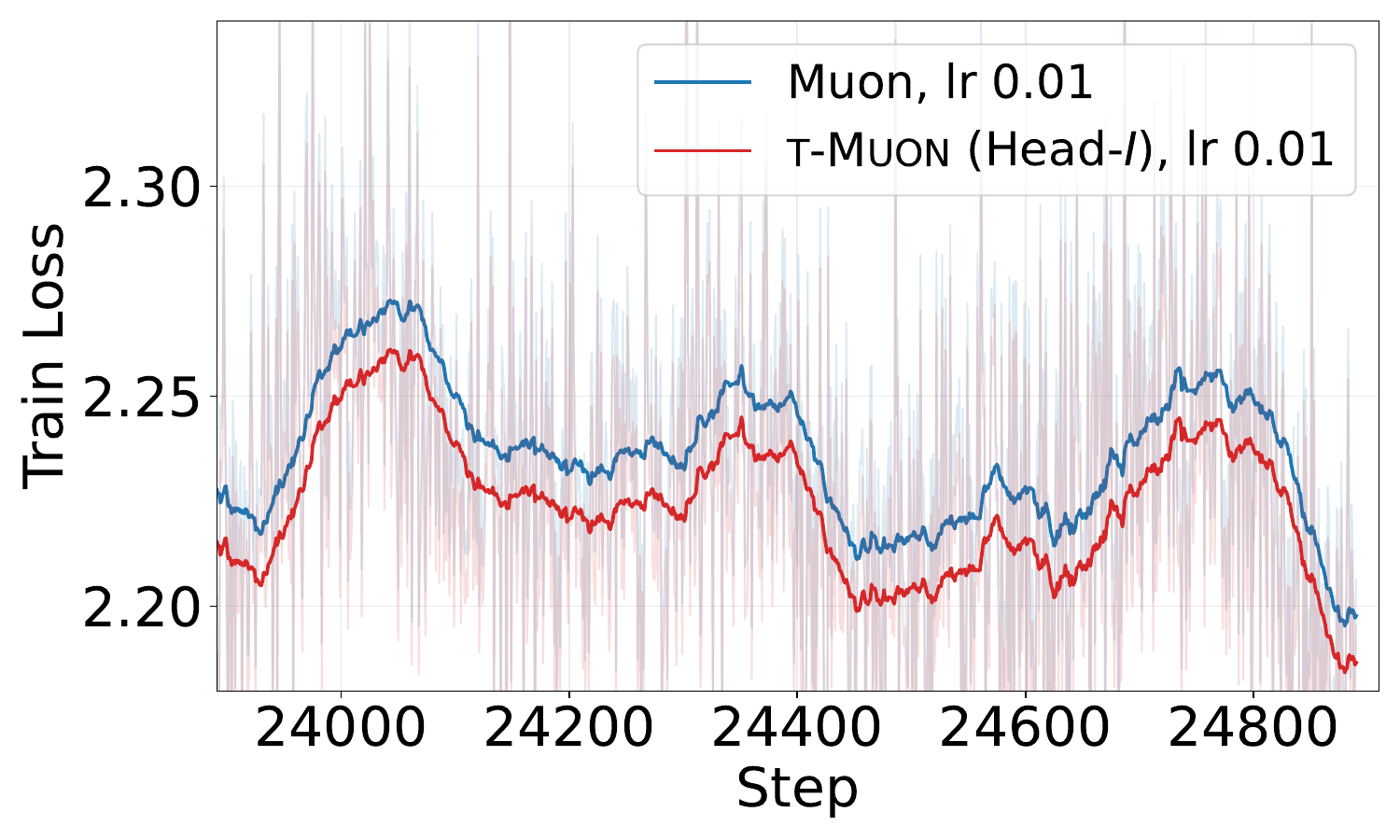}
        \caption{LLaMA-1B (13.1B tokens), late training}
        \label{fig:edu-1b-late}
    \end{subfigure}

    \caption{
        Training-loss curves of LLaMA-1B on 13.1B FineWeb-Edu tokens. Left: The first 1000 training steps. Right: the last 1000 training steps. Solid lines are the average over the last 50 steps.  
    }
    \label{fig:edu-train}
\end{figure*}
\subsection{Larger-Scale Pretraining on FineWeb-Edu}
\label{sec:fineweb-edu}

To test whether the gains persist at larger scale, we train LLaMA-60M, LLaMA-130M, LLaMA-350M, and LLaMA-1B on 1.2B, 2.7B, 7.4B, and 13.1B FineWeb-Edu tokens \citep{penedo2024fineweb}. The data are tokenized with the LLaMA-2 tokenizer; the other settings follow \Cref{tab:main-hyperparameters}. Due to limited computing resources, learning rates are swept at small scale and carried upward (\Cref{tab:lr-grids}): all optimizers are compared at 60M and 130M; AdamW, Muon, and the three best \textsc{t-Muon} variants at 130M are trained at 350M; and Muon and the best variant, Head-$I$, are trained at 1B. 

\begin{table}[t]
    \centering
    \setlength{\tabcolsep}{4pt}
    \begin{tabular}{@{}lcccc@{}}
        \toprule
        \textbf{Optimizer}
        & \textbf{60M}
        & \textbf{130M}
        & \textbf{350M}
        & \textbf{1B} \\
        \textit{Tokens}
        & \textit{1.2B}
        & \textit{2.7B}
        & \textit{7.4B}
        & \textit{13.1B} \\
        \midrule
        AdamW
        & 23.89
        & 16.59
        & 11.92
        & -- \\
        Muon
        & 20.58
        & 15.30
        & 11.60
        & 9.45 \\
        \midrule
        \multicolumn{5}{@{}l}{\textit{\textsc{t-Muon}}} \\
        Head-$I$
        & \textbf{20.23}
        & \textbf{15.09}
        & \textbf{11.41}
        & \textbf{9.34} \\
        Layer-$F$
        & 20.74
        & 15.44
        & 11.69
        & -- \\
        Head-stack-$S$
        & 20.65
        & 15.44
        & 11.69
        & -- \\
        \bottomrule
    \end{tabular}
    \caption{
        Validation perplexity of LLaMA-style models trained on FineWeb-Edu. Single runs at the selected learning rate; all variants are reported in \Cref{tab:edu-full}. The best result in each column is bold.
    }
    \label{tab:edu-main}
\end{table}

From \Cref{tab:edu-main}, Head-$I$ achieves lower validation perplexity at every scale and stays below Muon throughout training as shown in \Cref{fig:edu-train} . None of the DFT, DCT, or DST variants improves over Muon on FineWeb-Edu. The complete comparison is reported in \Cref{tab:edu-full} with single seed.

\section{Discussion and Conclusion}
\label{sec:discussion-conclusion}

This work extends Muon's matrix polar update to tensorized Transformer gradients through the $\mat{M}$-product and t-SVD. We show that the exact \textsc{t-Muon} direction is steepest descent under the $\star_{\mat{M}}$-spectral norm and establish a nonasymptotic stationarity guarantee for an idealized stochastic method with Nesterov-type momentum. In the FineWeb experiments, the identity-transform head variants achieve lower validation perplexity than AdamW and Muon on every evaluated GPT-2- and LLaMA-style model. The longer-context and longer-token-budget experiments further show that the improvement can persist beyond the shared one-billion-token setting. For FineWeb-Edu, Head-$I$ continues to improve on Muon in  pretraining up to 1.3B parameters. The identity transform is the strongest choice in all main comparisons and in the evaluated longer-token-budget experiments, although the best tensorization mode shifts among head, QKV, and layer groupings. The DFT Layer variant performs slightly better in the single-seed longer-context GPT-Base setting, providing one positive case for transform mixing. These findings motivate further work on selecting tensorization and transform choices for a given model and parameter group. In summary, \textsc{t-Muon} explores a principled t-SVD-based approach to Muon updates for tensor-valued gradient groups, and the evaluated variants show that architecture-aligned orthogonalization domains can improve language-model pretraining while leaving the choice of tensorization and transform as an open design problem.

\section*{Limitations}
This work has limitations in the evaluated transforms, the scope of the theory, the experimental scale, and the GPT-2 baseline. First, our experiments evaluate only fixed transforms: the identity, DFT, DCT, and DST. Our formulation and analysis apply to any fixed orthogonal transform. We do not evaluate gradient-dependent transforms inspired by data-adaptive mode-3 constructions for tensor compression \citep{kilmer2021tensor}. Second, our theoretical guarantees concern an idealized method with an exact polar map, a fixed transform, and a single parameter group. In contrast, the practical optimizer uses a polynomial polar approximation \citep{amsel2025polar} and combines tensorized, matrix, and AdamW updates across parameter groups. Third, our experiments cover models of up to 1.3B parameters and up to 13.1B training tokens, the longer-context and FineWeb-Edu comparisons use a single seed, and we evaluate validation loss and perplexity only, not downstream task performance. Finally, in the GPT-2 implementation, the query, key, and value projections form a single fused matrix, which the Muon baseline orthogonalizes as a whole, whereas every \textsc{t-Muon} mode first splits it into separate blocks. Unlike in the LLaMA-style models, whose projections are separate matrices, part of the gain over Muon on GPT-2 may therefore come from this split \citep{jordan2024muon}; Our GPT-2 experiments do not isolate this effect, so larger gains over fused-QKV Muon may partly reflect partitioning rather than transform mixing.
\bibliography{references}

@techreport{radford2019language,
  title       = {Language Models Are Unsupervised Multitask Learners},
  author      = {Radford, Alec and Wu, Jeffrey and Child, Rewon and Luan, David and Amodei, Dario and Sutskever, Ilya},
  institution = {OpenAI},
  year        = {2019},
  url        = {https://cdn.openai.com/better-language-models/language_models_are_unsupervised_multitask_learners.pdf}
}

@inproceedings{brown2020language,
  title     = {Language Models Are Few-Shot Learners},
  author    = {Brown, Tom B. and Mann, Benjamin and Ryder, Nick and Subbiah, Melanie and Kaplan, Jared D. and Dhariwal, Prafulla and Neelakantan, Arvind and Shyam, Pranav and Sastry, Girish and Askell, Amanda and Agarwal, Sandhini and Herbert-Voss, Ariel and Krueger, Gretchen and Henighan, Tom and Child, Rewon and Ramesh, Aditya and Ziegler, Daniel and Wu, Jeffrey and Winter, Clemens and Hesse, Chris and Chen, Mark and Sigler, Eric and Litwin, Mateusz and Gray, Scott and Chess, Benjamin and Clark, Jack and Berner, Christopher and McCandlish, Sam and Radford, Alec and Sutskever, Ilya and Amodei, Dario},
  booktitle = {Advances in Neural Information Processing Systems},
  volume    = {33},
  pages     = {1877--1901},
  publisher = {Curran Associates, Inc.},
  year      = {2020},
  url       = {https://proceedings.neurips.cc/paper_files/paper/2020/hash/1457c0d6bfcb4967418bfb8ac142f64a-Abstract.html}
}

@article{touvron2023llama,
  title   = {{LLaMA}: Open and Efficient Foundation Language Models},
  author  = {Touvron, Hugo and Lavril, Thibaut and Izacard, Gautier and Martinet, Xavier and Lachaux, Marie-Anne and Lacroix, Timoth{\'e}e and Rozi{\`e}re, Baptiste and Goyal, Naman and Hambro, Eric and Azhar, Faisal and Rodriguez, Aurelien and Joulin, Armand and Grave, Edouard and Lample, Guillaume},
  journal = {arXiv preprint arXiv:2302.13971},
  year    = {2023},
  url     = {https://arxiv.org/abs/2302.13971},
  doi        = {10.48550/arXiv.2302.13971}
}

@inproceedings{penedo2024fineweb,
  title     = {The {FineWeb} Datasets: Decanting the Web for the Finest Text Data at Scale},
  author    = {Penedo, Guilherme and Kydl{\'i}{\v{c}}ek, Hynek and Ben Allal, Loubna and Lozhkov, Anton and Mitchell, Margaret and Raffel, Colin and von Werra, Leandro and Wolf, Thomas},
  booktitle = {Advances in Neural Information Processing Systems},
  volume    = {37},
  publisher = {Curran Associates, Inc.},
  year      = {2024},
  note      = {Datasets and Benchmarks Track},
  doi       = {10.52202/079017-0970},
  url       = {https://proceedings.neurips.cc/paper_files/paper/2024/hash/370df50ccfdf8bde18f8f9c2d9151bda-Abstract-Datasets_and_Benchmarks_Track.html},
  pages      = {30811--30849}
}

@inproceedings{kingma2015adam,
  title     = {{Adam}: A Method for Stochastic Optimization},
  author    = {Kingma, Diederik P. and Ba, Jimmy},
  booktitle = {International Conference on Learning Representations},
  year      = {2015},
  url        = {https://arxiv.org/abs/1412.6980}
}

@inproceedings{loshchilov2019decoupled,
  title     = {Decoupled Weight Decay Regularization},
  author    = {Loshchilov, Ilya and Hutter, Frank},
  booktitle = {International Conference on Learning Representations},
  year      = {2019},
  url        = {https://openreview.net/forum?id=Bkg6RiCqY7}
}

@book{higham2008functions,
  title     = {Functions of Matrices: Theory and Computation},
  author    = {Higham, Nicholas J.},
  publisher = {Society for Industrial and Applied Mathematics},
  address   = {Philadelphia, PA},
  year      = {2008},
  doi       = {10.1137/1.9780898717778}
}

@article{kilmer2011factorization,
  title   = {Factorization Strategies for Third-Order Tensors},
  author  = {Kilmer, Misha E. and Martin, Carla D.},
  journal = {Linear Algebra and its Applications},
  volume  = {435},
  number  = {3},
  pages   = {641--658},
  year    = {2011},
  doi     = {10.1016/j.laa.2010.09.020}
}

@inproceedings{amsel2025polar,
  title     = {The {Polar Express}: Optimal Matrix Sign Methods and Their Application to the {Muon} Algorithm},
  author    = {Amsel, Noah and Persson, David and Musco, Christopher and Gower, Robert M.},
  booktitle = {International Conference on Learning Representations},
  year      = {2026},
  url        = {https://proceedings.iclr.cc/paper_files/paper/2026/hash/df88b275bef31ac96c85f0c4013734fc-Abstract-Conference.html}
}

@article{grishina2025accelerating,
  title   = {Accelerating {Newton--Schulz} Iteration for Orthogonalization via {Chebyshev}-Type Polynomials},
  author  = {Grishina, Ekaterina and Smirnov, Matvey and Rakhuba, Maxim},
  journal = {arXiv preprint arXiv:2506.10935},
  year    = {2025},
  url     = {https://arxiv.org/abs/2506.10935},
  doi        = {10.48550/arXiv.2506.10935}
}

@article{si2025adamuon,
  title   = {{AdaMuon}: Adaptive {Muon} Optimizer},
  author  = {Si, Chongjie and Zhang, Debing and Shen, Wei},
  journal = {arXiv preprint arXiv:2507.11005},
  year    = {2025},
  url     = {https://arxiv.org/abs/2507.11005},
  doi        = {10.48550/arXiv.2507.11005}
}

@article{chang2026muoneq,
  title   = {{MuonEq}: Balancing Before Orthogonalization with Lightweight Equilibration},
  author  = {Chang, Da and Shi, Qiankun and Zhang, Lvgang and Li, Yu and Zhang, Ruijie and Lu, Yao and Liu, Yongxiang and Yuan, Ganzhao},
  journal = {arXiv preprint arXiv:2603.28254},
  year    = {2026},
  url     = {https://arxiv.org/abs/2603.28254},
  doi        = {10.48550/arXiv.2603.28254}
}

@article{dolatabadi2026numuon,
  title   = {{NuMuon}: Nuclear-Norm-Constrained {Muon} for Compressible {LLM} Training},
  author  = {Dolatabadi, Hadi M. and Ajanthan, Thalaiyasingam and Ramasinghe, Sameera and Hewa Koneputugodage, Chamin P. and Siriwardhana, Shamane and Shevchenko, Violetta and Pajak, Karol and Snewin, James and Avraham, Gil and Long, Alexander},
  journal = {arXiv preprint arXiv:2603.03597},
  year    = {2026},
  url     = {https://arxiv.org/abs/2603.03597},
  doi        = {10.48550/arXiv.2603.03597}
}

@article{zhang2026groupmuon,
  title   = {When and Why Grouping Attention Heads Accelerates {Muon} Optimization},
  author  = {Zhang, Hongtao and Zhou, Wenjie and Chen, Wei and Cheng, Xueqi},
  journal = {arXiv preprint arXiv:2605.08933},
  year    = {2026},
  url     = {https://arxiv.org/abs/2605.08933},
  doi        = {10.48550/arXiv.2605.08933}
}

@inproceedings{khaled2025muonbp,
  title     = {{MuonBP}: Faster {Muon} via Block-Periodic Orthogonalization},
  author    = {Khaled, Ahmed and Ozkara, Kaan and Yu, Tao and Hong, Mingyi and Park, Youngsuk},
  booktitle = {International Conference on Learning Representations},
  year      = {2026},
  url        = {https://proceedings.iclr.cc/paper_files/paper/2026/hash/1d485d337d5622bdb1462333f4247522-Abstract-Conference.html},
  pages      = {17389--17416}
}

@article{li2026intrinsic,
  title   = {Intrinsic {Muon}: Spectral Optimization on {Riemannian} Matrix Manifolds},
  author  = {Li, Yibang and Pandey, Bihari Lal and Sah, Ravi and Han, Andi and Mostajeran, Cyrus and Jawanpuria, Pratik and Mishra, Bamdev},
  journal = {arXiv preprint arXiv:2605.09238},
  year    = {2026},
  url     = {https://arxiv.org/abs/2605.09238},
  doi        = {10.48550/arXiv.2605.09238}
}

@article{bernstein2024oldoptimizer,
  title   = {Old Optimizer, New Norm: An Anthology},
  author  = {Bernstein, Jeremy and Newhouse, Laker},
  journal = {arXiv preprint arXiv:2409.20325},
  year    = {2024},
  url     = {https://arxiv.org/abs/2409.20325},
  doi        = {10.48550/arXiv.2409.20325}
}

@misc{jordan2024muon,
  title  = {{Muon}: An Optimizer for Hidden Layers in Neural Networks},
  author = {Jordan, Keller and Jin, Yuchen and Boza, Vlado and You, Jiacheng and Cesista, Franz and Newhouse, Laker and Bernstein, Jeremy},
  year   = {2024},
  note   = {Blog post},
  url    = {https://kellerjordan.github.io/posts/muon/}
}

@article{kernfeld2015tensor,
  title   = {Tensor--Tensor Products with Invertible Linear Transforms},
  author  = {Kernfeld, Eric and Kilmer, Misha E. and Aeron, Shuchin},
  journal = {Linear Algebra and its Applications},
  volume  = {485},
  pages   = {545--570},
  year    = {2015},
  doi     = {10.1016/j.laa.2015.07.021}
}

@article{kilmer2021tensor,
  title   = {Tensor--Tensor Algebra for Optimal Representation and Compression of Multiway Data},
  author  = {Kilmer, Misha E. and Horesh, Lior and Avron, Haim and Newman, Elizabeth},
  journal = {Proceedings of the National Academy of Sciences of the United States of America},
  volume  = {118},
  number  = {28},
  pages   = {e2015851118},
  year    = {2021},
  doi     = {10.1073/pnas.2015851118}
}

@article{liu2025muon,
  title   = {{Muon} Is Scalable for {LLM} Training},
  author  = {Liu, Jingyuan and Su, Jianlin and Yao, Xingcheng and Jiang, Zhejun and Lai, Guokun and Du, Yulun and Qin, Yidao and Xu, Weixin and Lu, Enzhe and Yan, Junjie and Chen, Yanru and Zheng, Huabin and Liu, Yibo and Liu, Shaowei and Yin, Bohong and He, Weiran and Zhu, Han and Wang, Yuzhi and Wang, Jianzhou and Dong, Mengnan and Zhang, Zheng and Kang, Yongsheng and Zhang, Hao and Xu, Xinran and Zhang, Yutao and Wu, Yuxin and Zhou, Xinyu and Yang, Zhilin},
  journal = {arXiv preprint arXiv:2502.16982},
  year    = {2025},
  url     = {https://arxiv.org/abs/2502.16982},
  doi        = {10.48550/arXiv.2502.16982}
}

@article{shah2025muon,
  title   = {Practical Efficiency of {Muon} for Pretraining},
  author  = {Shah, Ishaan and Polloreno, Anthony M. and Stratos, Karl and Monk, Philip and Chaluvaraju, Adarsh and Hojel, Andrew and Ma, Andrew and Thomas, Anil and Tanwer, Ashish and Shah, Darsh J. and Nguyen, Khoi and Smith, Kurt and Callahan, Michael and Pust, Michael and Parmar, Mohit and Rushton, Peter and Mazarakis, Platon and Kapila, Ritvik and Srivastava, Saurabh and Singla, Somanshu and Romanski, Tim and Vanjani, Yash and Vaswani, Ashish},
  journal = {arXiv preprint arXiv:2505.02222},
  year    = {2025},
  url     = {https://arxiv.org/abs/2505.02222},
  doi        = {10.48550/arXiv.2505.02222}
}

@article{zhang2026teon,
  title   = {{TEON}: Tensorized Orthonormalization Beyond Layer-Wise {Muon} for Large Language Model Pre-Training},
  author  = {Zhang, Ruijie and Zhao, Yequan and Liu, Ziyue and Wang, Zhengyang and Li, Dongyang and Su, Yupeng and Liu, Sijia and Zhang, Zheng},
  journal = {arXiv preprint arXiv:2601.23261},
  year    = {2026},
  url     = {https://arxiv.org/abs/2601.23261},
  doi        = {10.48550/arXiv.2601.23261}
}

@article{bogachev2026tensorion,
  title   = {{Tensorion}: A Tensor-Aware Generalization of the {Muon} Optimizer},
  author  = {Bogachev, Vladimir and Aletov, Vladimir and Molozhavenko, Alexander and Kudriashov, Sergei and Rakhuba, Maxim},
  journal = {arXiv preprint arXiv:2606.25975},
  year    = {2026},
  url     = {https://arxiv.org/abs/2606.25975},
  doi        = {10.48550/arXiv.2606.25975}
}

@article{qing2026muonc,
  title   = {{Muon-C}: Operator-Aligned {Muon} for Convolutional Kernels},
  author  = {Qing, Jiaxin and Li, Lexin},
  journal = {arXiv preprint arXiv:2609.09676},
  year    = {2026},
  url     = {https://arxiv.org/abs/2609.09676},
  doi        = {10.48550/arXiv.2609.09676}
}

@article{kovalev2025understanding,
  title   = {Understanding Gradient Orthogonalization for Deep Learning via Non-{Euclidean} Trust-Region Optimization},
  author  = {Kovalev, Dmitry},
  journal = {arXiv preprint arXiv:2503.12645},
  year    = {2025},
  url     = {https://arxiv.org/abs/2503.12645},
  doi        = {10.48550/arXiv.2503.12645}
}

@article{li2025note,
  title   = {A Note on the Convergence of {Muon}},
  author  = {Li, Jiaxiang and Hong, Mingyi},
  journal = {arXiv preprint arXiv:2502.02900},
  year    = {2025},
  url     = {https://arxiv.org/abs/2502.02900},
  doi        = {10.48550/arXiv.2502.02900}
}

@article{martucci1994symmetric,
  title   = {Symmetric Convolution and the Discrete Sine and Cosine Transforms},
  author  = {Martucci, Stephen A.},
  journal = {IEEE Transactions on Signal Processing},
  volume  = {42},
  number  = {5},
  pages   = {1038--1051},
  year    = {1994},
  url        = {https://ieeexplore.ieee.org/document/295213}
}

@article{touvron2023llama2,
  title = {{Llama 2}: Open Foundation and Fine-Tuned Chat Models},
  author = {Hugo Touvron and Louis Martin and Kevin Stone and Peter Albert and Amjad Almahairi and Yasmine Babaei and Nikolay Bashlykov and Soumya Batra and Prajjwal Bhargava and Shruti Bhosale and Dan Bikel and Lukas Blecher and Cristian Canton Ferrer and Moya Chen and Guillem Cucurull and David Esiobu and Jude Fernandes and Jeremy Fu and Wenyin Fu and Brian Fuller and Cynthia Gao and Vedanuj Goswami and Naman Goyal and Anthony Hartshorn and Saghar Hosseini and Rui Hou and Hakan Inan and Marcin Kardas and Viktor Kerkez and Madian Khabsa and Isabel Kloumann and Artem Korenev and Punit Singh Koura and Marie-Anne Lachaux and Thibaut Lavril and Jenya Lee and Diana Liskovich and Yinghai Lu and Yuning Mao and Xavier Martinet and Todor Mihaylov and Pushkar Mishra and Igor Molybog and Yixin Nie and Andrew Poulton and Jeremy Reizenstein and Rashi Rungta and Kalyan Saladi and Alan Schelten and Ruan Silva and Eric Michael Smith and Ranjan Subramanian and Xiaoqing Ellen Tan and Binh Tang and Ross Taylor and Adina Williams and Jian Xiang Kuan and Puxin Xu and Zheng Yan and Iliyan Zarov and Yuchen Zhang and Angela Fan and Melanie Kambadur and Sharan Narang and Aurelien Rodriguez and Robert Stojnic and Sergey Edunov and Thomas Scialom},
  journal = {arXiv preprint arXiv:2307.09288},
  year = {2023},
  doi = {10.48550/arXiv.2307.09288},
  url = {https://arxiv.org/abs/2307.09288}
}

@inproceedings{pang2026htmuon,
  title     = "{HTMuon}: Improving {Muon} via Heavy-Tailed Spectral Correction",
  author    = "Pang, Tianyu and
               Fang, Yujie and
               Liu, Zihang and
               Deng, Shenyang and
               Hsiung, Lei and
               Yu, Shuhua and
               Yang, Yaoqing",
  editor    = "Liakata, Maria and
               Moreira, Viviane P. and
               Zhang, Jiajun and
               Jurgens, David",
  booktitle = "Findings of the Association for Computational Linguistics: {ACL} 2026",
  year      = "2026",
  month     = jul,
  address   = "San Diego, California, United States",
  publisher = "Association for Computational Linguistics",
  pages     = "36504--36535",
  doi       = "10.18653/v1/2026.findings-acl.1819",
  url       = "https://aclanthology.org/2026.findings-acl.1819/"
}

\appendix

\section{The Unitary DFT t-SVD}
\label{app:dft-extension}

We briefly state the complex extension of t-SVD when unitary DFT is used as transformation. For $\ten X\in\mathbb{C}^{m\times n\times K}$, define
\[
\widehat{\ten X} = \ten X\times_3\mat M.
\]
Following the transform-based tensor algebra of \citet{kernfeld2015tensor,kilmer2021tensor}, the $\mat M$-adjoint $\ten X^{*_{\mat M}}$ is defined by
\[
\widehat{\left(\ten X^{*_{\mat M}}\right)}_k = \widehat{\mat X}_k^{*}, \qquad k\in[K].
\]
Equivalently,
\[
\ten X^{*_{\mat M}} = \widehat{\ten X}^{\,*} \times_3 \mat M^{*},
\]
where $\widehat{\ten X}^{\,*}$ denotes conjugate transposition of each transformed frontal slice.

Accordingly, the $\mat M$-product t-SVD takes the form
\[
\ten X
=
\ten U
\star_{\mat M}
\ten S
\star_{\mat M}
\ten V^{*_{\mat M}},
\]
with transformed frontal slices
\[
\widehat{\mat X}_k
=
\widehat{\mat U}_k
\widehat{\mat S}_k
\widehat{\mat V}_k^{*},
\qquad k\in[K],
\]
where each factorization is a complex matrix SVD.

\section{Proof of the Steepest-Descent Direction}
\label{app:steepest-descent-proof}

This appendix proves \Cref{prop:star-m-steepest-descent}. We first consider a real orthogonal transform and then extend the result to the unitary DFT used in the experiments.

\subsection{Real Orthogonal Transforms}
\label{app:real-steepest-descent-proof}

\begin{proof}[Proof of \Cref{prop:star-m-steepest-descent}]

\Cref{eq:star-m-steepest-descent} is equivalent to
\begin{equation}
    \Delta\ten{W}^{\star}
    \in-\eta\cdot 
    \operatorname*{arg\,max}_{
        \|\ten{D}\|_{\star_{\mat{M}}}\leq 1
    }
    \left\langle
        \ten{G},
        \ten{D}
    \right\rangle .
\end{equation}
Because $\mat{M}$ is orthogonal, mode-three multiplication by $\mat{M}$ preserves the Frobenius inner product:
\begin{equation}
\begin{aligned}
    \left\langle
        \ten{G},
        \ten{D}
    \right\rangle
    &=
    \left\langle
        \widehat{\ten{G}},
        \widehat{\ten{D}}
    \right\rangle
    \\
    &=
    \sum_{k=1}^{K}
    \left\langle
        \widehat{\mat{G}}_k,
        \widehat{\mat{D}}_k
    \right\rangle.
\end{aligned}
\label{eq:app-transform-inner-product}
\end{equation}

By the definition of the $\star_{\mat{M}}$-spectral norm,
\begin{equation}
    \left\|
        \ten{D}
    \right\|_{\star_{\mat{M}}}
    \leq 1\Longleftrightarrow
    \left\|
        \widehat{\mat{D}}_k
    \right\|_{\mathrm{spec}}
    \leq 1,
    \label{eq:app-slice-constraints}
\end{equation} for all $k\in[K]$.

The linear optimization problem therefore separates across the transformed frontal slices:
\begin{equation}
\begin{aligned}
    &
    \max_{
        \|\ten{D}\|_{\star_{\mat{M}}}\leq 1
    }
    \left\langle
        \ten{G},
        \ten{D}
    \right\rangle
    \\
    &\quad=
    \sum_{k=1}^{K}
    \max_{
        \|\widehat{\mat{D}}_k\|_{\mathrm{spec}}\leq 1
    }
    \left\langle
        \widehat{\mat{G}}_k,
        \widehat{\mat{D}}_k
    \right\rangle.
\end{aligned}
\label{eq:app-separated-problem}
\end{equation}

For each $k\in[K]$, let
\begin{equation*}
    \widehat{\mat{G}}_k
    =
    \mat{U}_k
    \mat{\Sigma}_k
    \mat{V}_k^{\top}
\end{equation*}
be a thin matrix SVD. By spectral--nuclear norm duality,
\begin{equation}
    \max_{
        \|\widehat{\mat{D}}_k\|_{\mathrm{spec}}\leq 1
    }
    \left\langle
        \widehat{\mat{G}}_k,
       \widehat{\mat{D}}_k
    \right\rangle
    =
    \left\|
        \widehat{\mat{G}}_k
    \right\|_{\mathrm{nuc}},
    \label{eq:app-spectral-nuclear-duality}
\end{equation}
and one maximizer is
\begin{equation}
   \widehat{\mat{D}}_k^{\star}
    =
    \mat{U}_k\mat{V}_k^{\top}
    =
    \operatorname{Orth}
    \left(
        \widehat{\mat{G}}_k
    \right).
    \label{eq:app-slice-polar-solution}
\end{equation}

Let $\widehat{\ten{O}}$ contain these maximizers as frontal slices:
\begin{equation*}
    \widehat{\mat{O}}_k
    :=
    \widehat{\mat{D}}_k^{\star}.
\end{equation*}
Applying the inverse transform gives
\begin{equation*}
    \ten{O}
    =
    \widehat{\ten{O}}\times_3\mat{M}^{\top}
    =
    \operatorname{Orth}_{\mat{M}}(\ten{G}).
\end{equation*}
Substituting this maximizer into \Cref{eq:star-m-steepest-descent} yields
\begin{equation*}
    \Delta\ten{W}^{\star}
    =
    -\eta \cdot \operatorname{Orth}_{\mat{M}}(\ten{G}),
\end{equation*}
which proves the proposition.
\end{proof}

\subsection{The Unitary DFT Case}
\label{app:dft-steepest-descent}

Unitary $\mat M$ cannot split \Cref{eq:app-separated-problem} independently into subproblems. However, the DFT case differs from a general complex unitary transform because a real tensor has conjugate-symmetric DFT slices. This structure allows the real-constrained problem to decompose over conjugate pairs.

For complex matrices, define the Frobenius inner product by
\begin{equation}
    \left\langle
        \mat{A},
        \mat{B}
    \right\rangle_{\mathbb{C}}
    :=
    \operatorname{tr}
    \left(
        \mat{A}^{*}\mat{B}
    \right).
    \label{eq:app-complex-inner-product}
\end{equation}

\begin{proposition}[Steepest Descent under the Unitary DFT]
\label{prop:app-dft-steepest-descent}
Let $\mat{M}\in\mathbb{C}^{K\times K}$ be the unitary DFT matrix, and consider  $\ten{G}\in\mathbb{R}^{m\times n\times K}$. Let
\begin{equation}
    \widehat{\mat{O}}_k
    :=
    \operatorname{Orth}
    \left(
         \widehat{\mat{G}}_k
    \right),
    \qquad k\in[K].
    \label{eq:app-dft-polar-slices}
\end{equation}
Then
\begin{equation}
    \ten{O}
    :=
    \widehat{\ten{O}}\times_3\mat{M}^{*}
    \in
    \mathbb{R}^{m\times n\times K},
    \label{eq:app-dft-real-polar}
\end{equation}
and
\begin{equation}
    \Delta\ten{W}^{\star}
    = -\eta\cdot\ten{O},
    \label{eq:app-dft-steepest-direction}
\end{equation}
solves \Cref{eq:star-m-steepest-descent} under the unitary DFT transformation.
\end{proposition}

\begin{proof}
Again, we consider the problem 
\begin{equation}
    \Delta\ten{W}^{\star}
    \in-\eta\cdot 
    \operatorname*{arg\,max}_{
        \|\ten{D}\|_{\star_{\mat{M}}}\leq 1
    }
    \left\langle
        \ten{G},
        \ten{D}
    \right\rangle .
\end{equation}Because $\mat{M}$ is unitary, it preserves the Frobenius inner product.
For any real update tensor $\ten{D}$,
\begin{equation}
\begin{aligned}
    \left\langle
        \ten{G},
        \ten{D}
    \right\rangle
    &=
    \sum_{k=1}^{K}
    \left\langle
         \widehat{\mat{G}}_k,
        \widehat{\mat{D}}_k
    \right\rangle_{\mathbb{C}}.
\end{aligned}
\label{eq:app-dft-inner-product}
\end{equation}
Although individual terms may be complex, their sum is real.

Since $\ten{G}$ and $\ten{D}$ are real, their transformed tensors are conjugate symmetric. The slice $k=1$ is self-conjugate. When $K$ is even, the slice $k=K/2+1$ is also self-conjugate. For non-self-conjugate indices, define
\begin{equation*}
    \bar{k}
    :=
    K-k+2.
\end{equation*}
Then
\begin{equation}
    \widehat{\mat{G}}_{\bar{k}}
    =
    \overline{
         \widehat{\mat{G}}_k
    },
    \qquad
    \widehat{ \mat{D}}_{\bar{k}}
    =
    \overline{
        \widehat{\mat{D}}_k
    }.
    \label{eq:app-dft-conjugate-symmetry}
\end{equation}

For a non-self-conjugate pair $(k,\bar{k})$, the contribution to \Cref{eq:app-dft-inner-product} is
\begin{equation}
\begin{aligned}
    &
    \left\langle
         \widehat{\mat{G}}_k,
        \widehat{\mat{D}}_k
    \right\rangle_{\mathbb{C}} +
    \left\langle
        \widehat{\mat{G}}_{\bar{k}},
        \widehat{ \mat{D}}_{\bar{k}}
    \right\rangle_{\mathbb{C}}
    \\
    &=
    2\operatorname{Re}
    \left\langle
         \widehat{\mat{G}}_k,
        \widehat{\mat{D}}_k
    \right\rangle_{\mathbb{C}}.
\end{aligned}
\label{eq:app-dft-pair-objective}
\end{equation}
The spectral-norm constraints of the two slices are identical. If
\begin{equation*}
     \widehat{\mat{G}}_k
    =
    \mat{U}_k
    \mat{\Sigma}_k
    \mat{V}_k^{*},
\end{equation*}
complex spectral--nuclear norm duality gives the maximizer
\begin{equation}
    \widehat{\mat{O}}_k
    =
    \mat{U}_k\mat{V}_k^{*}.
    \label{eq:app-dft-pair-polar}
\end{equation}
The paired slice is chosen as
\begin{equation}
   \widehat{\mat{O}}_{\bar{k}}
    =
    \overline{
        \widehat{\mat{O}}_k
    }.
    \label{eq:app-dft-paired-polar}
\end{equation}
For each self-conjugate slice, the transformed gradient is real, and the ordinary real matrix polar factor solves the corresponding subproblem.

The resulting tensor $\widehat{\ten{O}}$ is conjugate symmetric. Consequently, its inverse DFT in \Cref{eq:app-dft-real-polar} is real. Combining the independently optimal self-conjugate slices and conjugate pairs proves \Cref{eq:app-dft-steepest-direction}.
\end{proof}

\section{Proof of the Convergence Guarantee}
\label{app:convergence-proof}

This appendix proves \Cref{thm:convergence}. Define
\begin{equation}
    \chi_{\mu}
    :=
    (1-\mu)
    \sqrt{
        2\mu
        +
        \frac{1}{1-\mu^{2}}
    }.
    \label{eq:app-chi-mu}
\end{equation}
The proof first establishes a one-step descent inequality, then bounds the gradient and momentum tracking errors, and finally combines these estimates over the optimization trajectory.

\subsection{Auxiliary Lemmas}
\label{app:convergence-lemmas}

\begin{lemma}[One-Step Descent Bound]
\label{lem:app-descent}
Suppose that \Cref{ass:smoothness} holds. The update in \Cref{eq:convergence-update} satisfies
\begin{equation}
\begin{aligned}
    f(\ten{W}_{t+1}) &\leq f(\ten{W}_{t}) - \eta \left\| \nabla f(\ten{W}_{t}) \right\|_{*}\\
    &\quad+ 2\eta \left\|\nabla f(\ten{W}_{t})-(1-\mu)\ten{Y}_{t}\right\|_{*} + \frac{\beta\eta^{2}}{2}.
\end{aligned}
\label{eq:app-descent-bound}
\end{equation}
\end{lemma}

\begin{proof}
By smoothness and $\ten{W}_{t+1}=\ten{W}_{t}-\eta\ten{D}_{t}$,
\begin{equation}
\begin{aligned}
    f(\ten{W}_{t+1}) &\leq f(\ten{W}_{t}) - \eta \left\langle \nabla f(\ten{W}_{t}), \ten{D}_{t} \right\rangle \\
    &\quad+ \frac{\beta\eta^{2}}{2} \left\| \ten{D}_{t}\right\|^{2}.
\end{aligned}
\label{eq:app-smoothness-step}
\end{equation}
Since $\|\ten{D}_{t}\|\leq 1$, the last term is at most $\beta\eta^{2}/2$.

Multiplying $\ten{Y}_t$ by the positive scalar $(1-\mu)$ does not change the maximizers of the linear oracle. Therefore,
\begin{equation}
    \left\langle
        (1-\mu)\ten{Y}_{t},
        \ten{D}_{t}
    \right\rangle
    =
    \left\|
        (1-\mu)\ten{Y}_{t}
    \right\|_{*}.
    \label{eq:app-lmo-optimality}
\end{equation}
Define the tracking error
\begin{equation}
    \ten{E}_{t}
    :=
    \nabla f(\ten{W}_{t})
    -
    (1-\mu)\ten{Y}_{t}.
    \label{eq:app-tracking-error}
\end{equation}
Then
\begin{equation}
\begin{aligned}
    &-
   \left\langle
        \nabla f(\ten{W}_{t}),
        \ten{D}_{t}
    \right\rangle\\
    =&
    -
    \left\|
        (1-\mu)\ten{Y}_{t}
    \right\|_{*}
    -
    \left\langle
        \ten{E}_{t},
        \ten{D}_{t}
    \right\rangle
    \\
    \leq&
    -
    \left\|
        (1-\mu)\ten{Y}_{t}
    \right\|_{*}
    +
    \|\ten{E}_{t}\|_{*}.
\end{aligned}
\label{eq:app-error-decomposition}
\end{equation}
The triangle inequality also gives
\begin{equation}
    \left\|
        \nabla f(\ten{W}_{t})
    \right\|_{*}
    \leq
    \left\|
        (1-\mu)\ten{Y}_{t}
    \right\|_{*}
    +
    \|\ten{E}_{t}\|_{*}.
    \label{eq:app-gradient-triangle}
\end{equation}
Combining \Cref{eq:app-error-decomposition,eq:app-gradient-triangle} yields
\begin{equation}
\begin{aligned}
    -
    \left\langle
        \nabla f(\ten{W}_{t}),
        \ten{D}_{t}
    \right\rangle
    &\leq
    -
    \left\|
        \nabla f(\ten{W}_{t})
    \right\|_{*}
    +
    2\|\ten{E}_{t}\|_{*}.
\end{aligned}
\end{equation}
Substituting this inequality into \Cref{eq:app-smoothness-step} proves the lemma.
\end{proof}

\begin{lemma}[Dual-Gradient Bound]
\label{lem:app-gradient-bound}
Suppose that \Cref{ass:smoothness} holds and that $f$ is bounded below by $f^{\star}$. Then every $\ten{W}\in\mathbb{R}^{m\times n\times K}$ satisfies
\begin{equation}
    \left\|
        \nabla f(\ten{W})
    \right\|_{*}^{2}
    \leq
    2\beta
    \left(
        f(\ten{W})-f^{\star}
    \right).
    \label{eq:app-gradient-bound}
\end{equation}
\end{lemma}

\begin{proof}
Let
\begin{equation}
    \ten{P}^{\star}
    \in
    \operatorname*{arg\,max}_{\|\ten{P}\|\leq 1}
    \left\langle
        \nabla f(\ten{W}),
        \ten{P}
    \right\rangle.
\end{equation}
Define
\begin{equation}
    \ten{W}^{+}
    :=
    \ten{W}
    -
    \frac{
        \|\nabla f(\ten{W})\|_{*}
    }{\beta}
    \ten{P}^{\star}.
    \label{eq:app-gradient-comparison-point}
\end{equation}
By smoothness,
\begin{equation}
\begin{aligned}
    f(\ten{W}^{+})
    \leq&
    f(\ten{W})
    +
    \left\langle
        \nabla f(\ten{W}),
        \ten{W}^{+}-\ten{W}
    \right\rangle
    \\&+
    \frac{\beta}{2}
    \left\|
        \ten{W}^{+}-\ten{W}
    \right\|^{2}
    \\
    \leq&
    f(\ten{W})
    -
    \frac{1}{2\beta}
    \left\|
        \nabla f(\ten{W})
    \right\|_{*}^{2}.
\end{aligned}
\label{eq:app-gradient-comparison}
\end{equation}
Since $f(\ten{W}^{+})\geq f^{\star}$, rearranging \Cref{eq:app-gradient-comparison} proves the result.
\end{proof}

\begin{lemma}[Momentum Tracking Error]
\label{lem:app-tracking-bound}
Suppose that \Cref{ass:smoothness,ass:stochastic-gradient,ass:norm-equivalence} hold. Then
\begin{equation}
\begin{aligned}
    \mathbb{E}&
    \left[
        \left\|
            \nabla f(\ten{W}_{t})
            -
            (1-\mu)\ten{Y}_{t}
        \right\|_{*}
    \right]\\
    &\leq
    \mu^{t+2}
    \sqrt{
        2\beta\Delta_{0}
    } +
    \frac{
        \beta\eta\mu^{2}
    }{
        1-\mu
    }
    +
    \rho\sigma\chi_{\mu}.
\end{aligned}
\label{eq:app-tracking-bound}
\end{equation}
\end{lemma}

\begin{proof}
Define the stochastic-gradient noise
\begin{equation}
    \ten{Z}_{t}
    :=
    \nabla f(\ten{W}_{t})
    -
    \ten{G}_{t},
    \label{eq:app-noise-definition}
\end{equation}
and, for $t\geq1$, define
\begin{equation}
    \ten{R}_{t}
    :=
    \nabla f(\ten{W}_{t})
    -
    \nabla f(\ten{W}_{t-1}).
    \label{eq:app-gradient-difference}
\end{equation}
Recall the tracking error
\begin{equation*}
    \ten{E}_{t}
    =
    \nabla f(\ten{W}_{t})
    -
    (1-\mu)\ten{Y}_{t}.
\end{equation*}

Using $\ten{Y}_{t}=\ten{G}_{t}+\mu\ten{B}_{t}$ and $\ten{B}_{t}=\mu\ten{B}_{t-1}+\ten{G}_{t}$ gives
\begin{equation}
    \ten{Y}_{t}
    =
    \mu\ten{Y}_{t-1}
    +
    (1+\mu)\ten{G}_{t}
    -
    \mu\ten{G}_{t-1}.
    \label{eq:app-direction-recursion}
\end{equation}
Substituting $\ten{G}_{t}=\nabla f(\ten{W}_{t})-\ten{Z}_{t}$ into \Cref{eq:app-direction-recursion} yields
\begin{equation}
\begin{aligned}
    \ten{E}_{t}
    =&\,
    \mu\ten{E}_{t-1}
    +
    \mu^{2}\ten{R}_{t}
    +
    (1-\mu^{2})\ten{Z}_{t}\\
    &-
    \mu(1-\mu)\ten{Z}_{t-1}.
\end{aligned}
\label{eq:app-error-recursion}
\end{equation}

Since $\ten{B}_{-1}=\mathbf{0}$,
\begin{equation}
    \ten{E}_{0}
    =
    \mu^{2}
    \nabla f(\ten{W}_{0})
    +
    (1-\mu^{2})\ten{Z}_{0}.
    \label{eq:app-initial-error}
\end{equation}
Unrolling \Cref{eq:app-error-recursion} gives
\begin{equation}
\begin{aligned}
    \ten{E}_{t}
    &=
    \mu^{t+2}
    \nabla f(\ten{W}_{0})
    +
    \mu^{2}
    \sum_{i=1}^{t}
    \mu^{t-i}
    \ten{R}_{i}
    +
    \ten{N}_{t},
\end{aligned}
\label{eq:app-error-unrolled}
\end{equation}
where
\begin{equation}
    \ten{N}_{t}
    :=
    (1-\mu^{2})\ten{Z}_{t}
    +
    (1-\mu)
    \sum_{i=0}^{t-1}
    \mu^{t-i+1}\ten{Z}_{i}.
    \label{eq:app-noise-aggregate}
\end{equation}

By \Cref{lem:app-gradient-bound},
\begin{equation}
    \left\|
        \mu^{t+2}
        \nabla f(\ten{W}_{0})
    \right\|_{*}
    \leq
    \mu^{t+2}
    \sqrt{
        2\beta\Delta_{0}
    }.
    \label{eq:app-initial-gradient-term}
\end{equation}
Smoothness and $\|\ten{W}_{i}-\ten{W}_{i-1}\|\leq\eta$ imply
\begin{equation}
    \|\ten{R}_{i}\|_{*}
    \leq
    \beta\eta.
\end{equation}
Therefore,
\begin{equation}
\begin{aligned}
    \left\|
        \mu^{2}
        \sum_{i=1}^{t}
        \mu^{t-i}\ten{R}_{i}
    \right\|_{*}
    &\leq
    \frac{
        \beta\eta\mu^{2}
    }{
        1-\mu
    }.
\end{aligned}
\label{eq:app-gradient-difference-bound}
\end{equation}

Under the independent mini-batch sampling and conditional unbiasedness in the stochastic-gradient model, distinct noise terms have zero expected cross inner product. Using \Cref{ass:stochastic-gradient},
\begin{equation}
\begin{aligned}
    &\mathbb{E}
    \left[
        \|\ten{N}_{t}\|_{\mathrm{F}}^{2}
    \right]\\
    &\leq
    \sigma^{2}
    \left[
        (1-\mu^{2})^{2}
        +
        (1-\mu)^{2}
        \sum_{i=0}^{t-1}
        \mu^{2(t-i+1)}
    \right]
    \\
    &\leq
    \sigma^{2}
    (1-\mu)^{2}
    \left(
        2\mu
        +
        \frac{1}{1-\mu^{2}}
    \right)
    \\
    &=
    \sigma^{2}\chi_{\mu}^{2}.
\end{aligned}
\label{eq:app-noise-second-moment}
\end{equation}
By Jensen's inequality and \Cref{ass:norm-equivalence}, 
\begin{equation}
\begin{aligned}
    \mathbb{E}
    \left[
        \|\ten{N}_{t}\|_{*}
    \right]
    &\leq
    \rho
    \sqrt{
        \mathbb{E}
        \left[
            \|\ten{N}_{t}\|_{\mathrm{F}}^{2}
        \right]
    }
    \\
    &\leq
    \rho\sigma\chi_{\mu}.
\end{aligned}
\label{eq:app-noise-dual-bound}
\end{equation}
Applying the triangle inequality to
\Cref{eq:app-error-unrolled} and combining
\Cref{eq:app-initial-gradient-term,eq:app-gradient-difference-bound,eq:app-noise-dual-bound}
proves the lemma.
\end{proof}

\subsection{Proof of \Cref{thm:convergence}}
\label{app:proof-convergence-theorem}

\begin{proof}[Proof of \Cref{thm:convergence}]
Taking expectations in
\Cref{eq:app-descent-bound} gives
\begin{equation}
\begin{aligned}
    \mathbb{E}
    \left[
        f(\ten{W}_{t+1})
    \right]
    &\leq
    \mathbb{E}
    \left[
        f(\ten{W}_{t})
    \right]
    -
    \eta
    \mathbb{E}
    \left[
        \|\nabla f(\ten{W}_{t})\|_{*}
    \right]
    \\
    &\quad+
    2\eta
    \mathbb{E}
    \left[
        \|\ten{E}_{t}\|_{*}
    \right]
    +
    \frac{\beta\eta^{2}}{2}.
\end{aligned}
\label{eq:app-expected-descent}
\end{equation}
Applying \Cref{lem:app-tracking-bound} and rearranging yields
\begin{equation}
\begin{aligned}
    \mathbb{E}
    \left[
        \|\nabla f(\ten{W}_{t})\|_{*}
    \right]
    &\leq
    \frac{
        \mathbb{E}[f(\ten{W}_{t})]
        -
        \mathbb{E}[f(\ten{W}_{t+1})]
    }{\eta}
    \\
    &\quad+
    2\mu^{t+2}
    \sqrt{
        2\beta\Delta_{0}
    }
    +
    \frac{
        2\beta\eta\mu^{2}
    }{
        1-\mu
    }
    \\
    &\quad+
    2\rho\sigma\chi_{\mu}
    +
    \frac{\beta\eta}{2}.
\end{aligned}
\label{eq:app-per-iteration-gradient}
\end{equation}

Summing \Cref{eq:app-per-iteration-gradient} from $t=0$ to $T_{\mathrm{iter}}-1$, the objective differences telescope:
\begin{equation}
\begin{aligned}
    &
    \sum_{t=0}^{T_{\mathrm{iter}}-1}
    \left(
        \mathbb{E}[f(\ten{W}_{t})]
        -
        \mathbb{E}[f(\ten{W}_{t+1})]
    \right)
    \\
    &=
    f(\ten{W}_{0})
    -
    \mathbb{E}
    \left[
        f(\ten{W}_{T_{\mathrm{iter}}})
    \right]
    \leq
    \Delta_{0}.
\end{aligned}
\label{eq:app-objective-telescope}
\end{equation}
Moreover,
\begin{equation}
    \sum_{t=0}^{T_{\mathrm{iter}}-1}
    \mu^{t+2}
    \leq
    \frac{\mu^{2}}{1-\mu}.
    \label{eq:app-momentum-geometric-sum}
\end{equation}
Consequently,
\begin{equation}
\begin{aligned}
    &
    \frac{1}{T_{\mathrm{iter}}}
    \sum_{t=0}^{T_{\mathrm{iter}}-1}
    \mathbb{E}
    \left[
        \|\nabla f(\ten{W}_{t})\|_{*}
    \right]
    \\
    &\leq
    \frac{
        \Delta_{0}
    }{
        \eta T_{\mathrm{iter}}
    }
    +
    \frac{
        2\mu^{2}
        \sqrt{
            2\beta\Delta_{0}
        }
    }{
        (1-\mu)T_{\mathrm{iter}}
    }
    \\
    &\quad+
    \frac{
        2\beta\eta\mu^{2}
    }{
        1-\mu
    }
    +
    2\rho\sigma\chi_{\mu}
    +
    \frac{\beta\eta}{2}.
\end{aligned}
\label{eq:app-average-gradient-bound}
\end{equation}
Finally, for every realization,
\begin{equation}
    \min_{0\leq t<T_{\mathrm{iter}}}
    \|\nabla f(\ten{W}_{t})\|_{*}
    \leq
    \frac{1}{T_{\mathrm{iter}}}
    \sum_{t=0}^{T_{\mathrm{iter}}-1}
    \|\nabla f(\ten{W}_{t})\|_{*}.
\end{equation}
Taking expectations and applying
\Cref{eq:app-average-gradient-bound} proves
\Cref{eq:convergence-bound}.
\end{proof}

\section{Model Architectures and Training Hyper-parameters}
\label{app:experimental-details}

\subsection{Model Configurations and Parameter Counts}
\label{app:model-configurations}

\Cref{tab:model-configurations} summarizes the architectures of the eight decoder-only models used in the experiments.

\begin{table*}[t]
    \centering
    \small
    \setlength{\tabcolsep}{6pt}
    \begin{tabular}{llrrrrrr}
        \toprule
        \textbf{Model}
        & $L$
        & $d$
        & $H$
        & $d_{\mathrm{ff}}$
        & \textbf{Params. (M)} \\
        \midrule
        GPT-50M
        & 8
        & 512
        & 8
        & 2,048
        & 51.475 \\
        GPT-Small
        & 12
        & 768
        & 12
        & 3,072
        & 124.439 \\
        GPT-Base
        & 24
        & 1,024
        & 16
        & 4,096
        & 354.822 \\
        GPT-Large
        & 36
        & 1,280
        & 20
        & 5,120
        & 774.029 \\
        \midrule
        LLaMA-60M
        & 8
        & 512
        & 8
        & 1,376
        & 58.074 \\
        LLaMA-130M
        & 12
        & 768
        & 12
        & 2,048
        & 134.106 \\
        LLaMA-350M
        & 24
        & 1,024
        & 16
        & 2,736
        & 367.969 \\
        LLaMA-1B
        & 24
        & 2,048
        & 32
        & 5,461
        & 1,339.083 \\
        \bottomrule
    \end{tabular}
    \caption{
        Architecture configurations used of GPT and LLaMA style models. Here, $L$, $d$, $H$, and $d_{\mathrm{ff}}$ denote the number of layers, hidden dimension, number of heads, and MLP intermediate dimension, respectively.
    }
    \label{tab:model-configurations}
\end{table*}

\subsection{Training and Optimizer Hyperparameters}
\label{app:training-hyperparameters}

Both the FineWeb and FineWeb-Edu experiments use a context length of 1,024, a batch size of 512, and bfloat16 mixed precision; the shared training hyperparameters are listed in \Cref{tab:main-hyperparameters}. Unless otherwise stated, runs use seeds 42, 43, and 44. The FineWeb-Edu experiments (\Cref{sec:fineweb-edu}) use the single seed 42 and report the final loss on a held-out validation split of 100M tokens. Each run uses four NVIDIA A100 GPUs, except the LLaMA-1B runs, which use 32 GPUs with the same global batch size.

\begin{table*}[t]
    \centering
    \small
    \setlength{\tabcolsep}{6pt}
    \begin{tabular}{p{0.3\textwidth}p{0.3\textwidth}}
        \toprule
        \textbf{Hyperparameter} &  \\
        \midrule
        Context length
        & 1,024 \\
        Batch size
        & 512 \\
        LR schedule
        & Cosine \\
        Warmup ratio
        & $0.1$ \\
        Weight decay
        & $0.1$ \\
        Polar approximation
        & PolarExpress \citep{amsel2025polar} \\
        Number of polar iterations
        & $5$ \\
        \bottomrule
    \end{tabular}
    \caption{
        Training hyperparameters used in the main experiments.
    }
    \label{tab:main-hyperparameters}
\end{table*}

\section{Additional Experimental Results}
\label{app:additional-results}

This appendix reports the complete optimizer comparisons underlying the compact tables in the main text.

\Cref{tab:gpt-ppl-full} extends \Cref{tab:gpt-ppl-main} to all fourteen \textsc{t-Muon} variants on the four GPT-2-style models.

\begin{table*}[t]
    \centering
    \footnotesize
    \setlength{\tabcolsep}{4.3pt}
    \begin{tabular}{lcccc}
        \toprule
        \textbf{Optimizer}
        & \textbf{GPT-50M}
        & \textbf{GPT-Small}
        & \textbf{GPT-Base}
        & \textbf{GPT-Large} \\
        \midrule
        AdamW
        & 47.35$\pm$0.081 (0.005)
        & 42.62$\pm$0.083 (0.005)
        & 38.64$\pm$0.081 (0.001)
        & 38.06$\pm$0.121 (0.001) \\
        Muon
        & 42.24$\pm$0.118 (0.02)
        & 35.48$\pm$0.052 (0.01)
        & 30.35$\pm$0.024 (0.01)
        & 27.95$\pm$0.046 (0.02) \\
        \midrule
        \multicolumn{5}{l}{\textit{\textsc{t-Muon} (proposed)}} \\
        Head-$F$
        & 42.00$\pm$0.079 (0.01)
        & 35.23$\pm$0.064 (0.01)
        & 30.41$\pm$0.048 (0.01)
        & 28.30$\pm$0.009 (0.01) \\
        Head-$C$
        & 41.73$\pm$0.036 (0.01)
        & 35.05$\pm$0.018 (0.01)
        & 30.30$\pm$0.019 (0.01)
        & 28.21$\pm$0.063 (0.01) \\
        Head-$S$
        & 41.83$\pm$0.019 (0.01)
        & 35.09$\pm$0.051 (0.01)
        & 30.34$\pm$0.024 (0.01)
        & 28.22$\pm$0.023 (0.01)\\
        Head-$I$
        & 40.57$\pm$0.100 (0.02)
        & 33.89$\pm$0.079 (0.02)
        & 29.07$\pm$0.018 (0.02)
        & \textbf{26.96$\pm$0.020} (0.02) \\
        Head-stack-$F$
        & 43.19$\pm$0.121 (0.02)
        & 36.28$\pm$0.145 (0.01)
        & 31.01$\pm$0.105 (0.04)
        & 28.75$\pm$0.038 (0.01) \\
        Head-stack-$C$
        & 43.11$\pm$0.153 (0.01)
        & 36.17$\pm$0.165 (0.02)
        & 30.80$\pm$0.179 (0.04)
        & 28.76$\pm$0.075 (0.01) \\
        Head-stack-$S$
        & 43.13$\pm$0.173 (0.02)
        & 36.15$\pm$0.073 (0.02)
        & 30.70$\pm$0.290 (0.04)
        & 28.73$\pm$0.062 (0.01)\\
        Head-stack-$I$
        & \textbf{40.52$\pm$0.069} (0.02)
        & \textbf{33.77$\pm$0.023} (0.02)
        & \textbf{29.07$\pm$0.035} (0.02)
        & 26.97$\pm$0.043 (0.02) \\
        Layer-$F$
        & 41.57$\pm$0.041 (0.01)
        & 34.68$\pm$0.031 (0.01)
        & 29.75$\pm$0.017 (0.01)
        & 27.36$\pm$0.005 (0.01) \\
        Layer-$I$
        & 40.61$\pm$0.005 (0.01)
        & 34.11$\pm$0.021 (0.01)
        & 29.39$\pm$0.014 (0.01)
        & 27.13$\pm$0.023 (0.01) \\
        QKV-$F$
        & 41.93$\pm$0.075 (0.02)
        & 35.34$\pm$0.027 (0.01)
        & 30.16$\pm$0.030 (0.02)
        & 27.74$\pm$0.017 (0.01) \\
        QKV-$C$
        & 41.71$\pm$0.121 (0.02)
        & 35.15$\pm$0.012 (0.01)
        & 30.14$\pm$0.048 (0.02)
        & 27.67$\pm$0.011 (0.01) \\
        QKV-$S$
        & 41.58$\pm$0.165 (0.02)
        & 35.10$\pm$0.053 (0.01)
        & 30.10$\pm$0.054 (0.02)
        & 27.65$\pm$0.018 (0.01) \\
        QKV-$I$
        & 40.60$\pm$0.029 (0.01)
        & 34.12$\pm$0.025 (0.01)
        & 29.38$\pm$0.014 (0.01)
        & 27.12$\pm$0.017 (0.01) \\
        \bottomrule
    \end{tabular}
    \caption{
        Complete validation-perplexity comparison on GPT-2-style models trained on 1B FineWeb tokens. Each entry reports mean $\pm$ standard deviation over three seeds with the best learning rate in parentheses. The best mean in each column is bold.}
    \label{tab:gpt-ppl-full}
\end{table*}

\Cref{tab:llama-ppl-full} reports the corresponding comparison on the LLaMA-style models.

\begin{table*}[t]
    \centering
    \footnotesize
    \setlength{\tabcolsep}{7pt}
    \begin{tabular}{@{}lccc@{}}
        \toprule
        \textbf{Optimizer}
        & \textbf{LLaMA-60M}
        & \textbf{LLaMA-130M}
        & \textbf{LLaMA-350M} \\
        \midrule
        AdamW
        & 32.17$\pm$0.062 (0.005)
        & 28.15$\pm$0.025 (0.001)
        & 24.03$\pm$0.061 (0.001) \\
        Muon
        & 26.56$\pm$0.022 (0.04)
        & 22.36$\pm$0.012 (0.04)
        & 19.39$\pm$0.003 (0.02) \\
        \midrule
        \multicolumn{4}{@{}l}{\textit{\textsc{t-Muon} (proposed)}} \\
        Head-$F$
        & 26.69$\pm$0.003 (0.04)
        & 22.40$\pm$0.023 (0.04)
        & 19.49$\pm$0.006 (0.02) \\
        Head-$C$
        & 26.40$\pm$0.004 (0.04)
        & 22.31$\pm$0.006 (0.04)
        & 19.45$\pm$0.006 (0.02) \\
        Head-$S$
        & 26.43$\pm$0.029 (0.04)
        & 22.31$\pm$0.013 (0.04)
        & 19.44$\pm$0.004 (0.02) \\
        Head-$I$
        & 25.95$\pm$0.025 (0.04)
        & \textbf{21.92$\pm$0.003} (0.02)
        & \textbf{19.06$\pm$0.004} (0.02) \\
        Head-stack-$F$
        & 26.68$\pm$0.042 (0.04)
        & 22.39$\pm$0.012 (0.04)
        & 19.49$\pm$0.005 (0.02) \\
        Head-stack-$C$
        & 26.41$\pm$0.015 (0.04)
        & 22.32$\pm$0.011 (0.04)
        & 19.44$\pm$0.004 (0.02) \\
        Head-stack-$S$
        & 26.42$\pm$0.034 (0.04)
        & 22.29$\pm$0.003 (0.04)
        & 19.45$\pm$0.008 (0.02) \\
        Head-stack-$I$
        & \textbf{25.93$\pm$0.007} (0.04)
        & \textbf{21.92$\pm$0.009} (0.02)
        & 19.06$\pm$0.008 (0.02) \\
        Layer-$F$
        & 26.70$\pm$0.035 (0.04)
        & 22.39$\pm$0.011 (0.04)
        & 19.47$\pm$0.003 (0.02) \\
        Layer-$I$
        & 26.59$\pm$0.010 (0.04)
        & 22.34$\pm$0.005 (0.04)
        & 19.40$\pm$0.004 (0.02) \\
        QKV-$F$
        & 27.47$\pm$0.008 (0.04)
        & 23.32$\pm$0.017 (0.02)
        & 20.17$\pm$0.016 (0.02) \\
        QKV-$C$
        & 27.26$\pm$0.079 (0.04)
        & 23.19$\pm$0.026 (0.02)
        & 20.01$\pm$0.011 (0.02) \\
        QKV-$S$
        & 27.24$\pm$0.036 (0.04)
        & 23.13$\pm$0.017 (0.02)
        & 19.98$\pm$0.008 (0.02) \\
        QKV-$I$
        & 26.57$\pm$0.009 (0.04)
        & 22.36$\pm$0.017 (0.04)
        & 19.38$\pm$0.012 (0.02) \\
        \bottomrule
    \end{tabular}
    \caption{Complete validation-perplexity comparison on LLaMA-style models trained on 1B FineWeb tokens. Each entry reports mean $\pm$ standard deviation over three seeds with the best learning rate in parentheses. The best mean in each column is bold.}
    \label{tab:llama-ppl-full}
\end{table*}

\Cref{tab:long-budget-full} reports the full comparison under larger token budgets.

\begin{table*}[t]
    \centering
    \footnotesize
    \setlength{\tabcolsep}{7pt}
    \begin{tabular}{lccc}
        \toprule
        \textbf{Optimizer}
        & \textbf{GPT-Small, 3B}
        & \textbf{GPT-Base, 7B}
        & \textbf{LLaMA-130M, 2.2B} \\
        \midrule
        AdamW
        & 28.97$\pm$0.018 (0.01)
        & 20.96$\pm$0.135 (0.005)
        & 21.51$\pm$0.033 (0.005) \\
        Muon
        & 27.94$\pm$0.046 (0.02)
        & 20.47$\pm$0.128 (0.01)
        & 19.35$\pm$0.013 (0.02) \\
        \midrule
        \multicolumn{4}{l}{\textit{\textsc{t-Muon} (proposed)}} \\
        Head-$F$
        & 27.62$\pm$0.021 (0.02)
        & 20.12$\pm$0.110 (0.01)
        & 19.58$\pm$0.006 (0.04) \\
        Head-$I$
        & 27.09$\pm$0.014 (0.02)
        & 19.87$\pm$0.112 (0.01)
        & \textbf{19.00$\pm$0.003} (0.02) \\
        Head-stack-$F$
        & 28.33$\pm$0.050 (0.01)
        & 20.51$\pm$0.093 (0.01)
        & 19.59$\pm$0.002 (0.04) \\
        Head-stack-$I$
        & 27.08$\pm$0.013 (0.02)
        & 19.87$\pm$0.098 (0.01)
        & 19.01$\pm$0.003 (0.02) \\
        Layer-$F$
        & 27.33$\pm$0.012 (0.01)
        & 19.97$\pm$0.068 (0.01)
        & 19.58$\pm$0.008 (0.02) \\
        Layer-$I$
        & 27.04$\pm$0.016 (0.02)
        & \textbf{19.85$\pm$0.085} (0.01)
        & 19.35$\pm$0.001 (0.02) \\
        QKV-$F$
        & 27.74$\pm$0.012 (0.02)
        & 20.30$\pm$0.076 (0.01)
        & 19.91$\pm$0.009 (0.02) \\
        QKV-$I$
        & \textbf{27.03$\pm$0.011} (0.02)
        & 19.85$\pm$0.084 (0.01)
        & 19.34$\pm$0.006 (0.02) \\
        \bottomrule
    \end{tabular}
    \caption{ Complete validation-perplexity comparison under larger token budgets. Each entry reports mean $\pm$ standard deviation over three seeds with the best learning rate in parentheses. The best mean in each column is bold.}
    \label{tab:long-budget-full}
\end{table*}

\Cref{tab:long-context-full} reports the full single-seed comparison for the longer-context experiments.

\begin{table*}[t]
    \centering
    \footnotesize
    \setlength{\tabcolsep}{7pt}
    \begin{tabular}{lccc}
        \toprule
        \textbf{Optimizer}
        & \textbf{GPT-Small}
        & \textbf{GPT-Base}
        & \textbf{LLaMA-350M} \\
        \midrule
        Context length
        & 2048
        & 4096
        & 4096 \\
        \midrule
        AdamW
        & 65.99 (0.001)
        & 221.50 (0.0005)
        & 70.25 (0.001) \\
        Muon
        & 41.87 (0.01)
        & 55.90 (0.04)
        & 28.71 (0.02) \\
        \midrule
        \multicolumn{4}{l}{\textit{\textsc{t-Muon} (proposed)}} \\
        Head-$F$
        & 39.82 (0.04)
        & 55.45 (0.005)
        & 29.42 (0.02) \\
        Head-$I$
        & \textbf{36.91} (0.04)
        & 51.84 (0.01)
        & 27.35 (0.02) \\
        Head-stack-$F$
        & 41.68 (0.02)
        & 60.74 (0.005)
        & 29.57 (0.02) \\
        Head-stack-$I$
        & 37.14 (0.04)
        & 51.94 (0.01)
        & \textbf{27.31} (0.02) \\
        Layer-$F$
        & 37.96 (0.04)
        & \textbf{49.77} (0.02)
        & 28.61 (0.02) \\
        Layer-$I$
        & 37.91 (0.04)
        & 49.97 (0.02)
        & 28.55 (0.02) \\
        QKV-$F$
        & 41.14 (0.02)
        & 55.63 (0.02)
        & 39.82 (0.02) \\
        QKV-$I$
        & 38.15 (0.02)
        & 50.04 (0.02)
        & 28.63 (0.02) \\
        \bottomrule
    \end{tabular}
    \caption{ Complete single-seed validation-perplexity comparison for the longer-context experiments. Results use seed 42. The best learning rate is shown in parentheses. The best value in each column is bold.}
    \label{tab:long-context-full}
\end{table*}

\Cref{tab:edu-full} extends \Cref{tab:edu-main} to all evaluated optimizers, reporting both the final validation loss and perplexity.

\begin{table*}[t]
    \centering
    \footnotesize
    \setlength{\tabcolsep}{5pt}
    \begin{tabular}{@{}lcccc@{}}
        \toprule
        \textbf{Optimizer}
        & \textbf{LLaMA-60M, 1.2B}
        & \textbf{LLaMA-130M, 2.7B}
        & \textbf{LLaMA-350M, 7.4B}
        & \textbf{LLaMA-1B, 13.1B} \\
        \midrule
        AdamW
        & 3.1737 / 23.89 (0.005)
        & 2.8086 / 16.59 (0.005)
        & 2.4782 / 11.92 (0.005)
        & -- \\
        Muon
        & 3.0241 / 20.58 (0.04)
        & 2.7277 / 15.30 (0.02)
        & 2.4509 / 11.60 (0.01)
        & 2.2457 / 9.45 (0.01) \\
        \midrule
        \multicolumn{5}{@{}l}{\textit{\textsc{t-Muon} (proposed)}} \\
        Head-$I$
        & \textbf{3.0071 / 20.23} (0.04)
        & \textbf{2.7140 / 15.09} (0.02)
        & \textbf{2.4349 / 11.41} (0.01)
        & \textbf{2.2339 / 9.34} (0.01) \\
        Head-$F$
        & 3.0293 / 20.68 (0.04)
        & 2.7426 / 15.53 (0.04)
        & --
        & -- \\
        Head-$C$
        & 3.0282 / 20.66 (0.04)
        & 2.7381 / 15.46 (0.02)
        & --
        & -- \\
        Head-$S$
        & 3.0264 / 20.62 (0.04)
        & 2.7376 / 15.45 (0.04)
        & --
        & -- \\
        Head-stack-$F$
        & 3.0318 / 20.73 (0.04)
        & 2.7415 / 15.51 (0.04)
        & --
        & -- \\
        Head-stack-$C$
        & 3.0273 / 20.64 (0.04)
        & 2.7375 / 15.45 (0.04)
        & --
        & -- \\
        Head-stack-$S$
        & 3.0279 / 20.65 (0.04)
        & 2.7370 / 15.44 (0.04)
        & 2.4589 / 11.69 (0.02)
        & -- \\
        Layer-$F$
        & 3.0320 / 20.74 (0.04)
        & 2.7370 / 15.44 (0.02)
        & 2.4588 / 11.69 (0.01)
        & -- \\
        QKV-$F$
        & 3.0437 / 20.98 (0.04)
        & 2.7484 / 15.62 (0.02)
        & --
        & -- \\
        QKV-$C$
        & 3.0430 / 20.97 (0.04)
        & 2.7471 / 15.60 (0.02)
        & --
        & -- \\
        QKV-$S$
        & 3.0395 / 20.89 (0.04)
        & 2.7489 / 15.63 (0.02)
        & --
        & -- \\
        \bottomrule
    \end{tabular}
    \caption{
        Complete FineWeb-Edu comparison. Each entry reports the final validation loss / perplexity of a single seed, with the selected learning rate in parentheses; the best result in each column is bold.
    }
    \label{tab:edu-full}
\end{table*}

\subsection{Learning-Rate Sensitivity}
\label{app:learning-rate-sensitivity}

\Cref{tab:lr-grids} lists the learning-rate grid swept for each optimizer. For every optimizer--model pair, we select the rate with the lowest mean final validation perplexity over the three seeds; the selected values are shown in parentheses in \Cref{tab:gpt-ppl-full,tab:llama-ppl-full,tab:long-budget-full}. \Cref{fig:lr-sweep} plots the corresponding sweeps for 1B FineWeb experiments. In the main comparisons, every identity-transform \textsc{t-Muon} variant selects a rate in the interior of its grid.

\begin{table*}[t]
    \centering
    \small
    \setlength{\tabcolsep}{6pt}
    \begin{tabular}{@{}llc@{}}
        \toprule
        \textbf{Optimizer}
        & \textbf{Models}
        & \textbf{Learning-rate grid} \\
        \midrule
        AdamW
        & All
        & $\{5{\times}10^{-4},\,10^{-3},\,5{\times}10^{-3},\,10^{-2}\}$ \\
        Muon and \textsc{t-Muon}
        & GPT-2-style
        & $\{5{\times}10^{-3},\,10^{-2},\,2{\times}10^{-2},\,4{\times}10^{-2}\}$ \\
        Muon and \textsc{t-Muon}
        & LLaMA-style
        & $\{5{\times}10^{-3},\,10^{-2},\,2{\times}10^{-2},\,4{\times}10^{-2},\,6{\times}10^{-2},\,8{\times}10^{-2}\}$ \\
        \bottomrule
    \end{tabular}
    \caption{
        Learning-rate grids swept for each optimizer.
    }
    \label{tab:lr-grids}
\end{table*}

\begin{figure*}[t]
    \centering
    \begin{subfigure}[t]{0.48\textwidth}
        \centering
        \includegraphics[width=\linewidth]{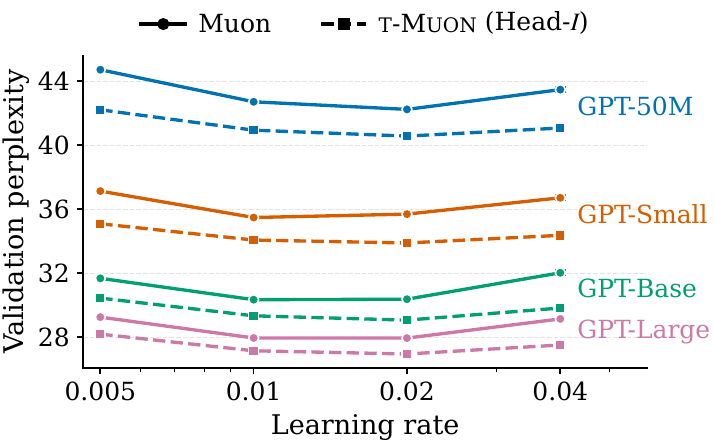}
        \caption{GPT-2-style models}
        \label{fig:lr-sweep-gpt}
    \end{subfigure}
    \hfill
    \begin{subfigure}[t]{0.48\textwidth}
        \centering
        \includegraphics[width=\linewidth]{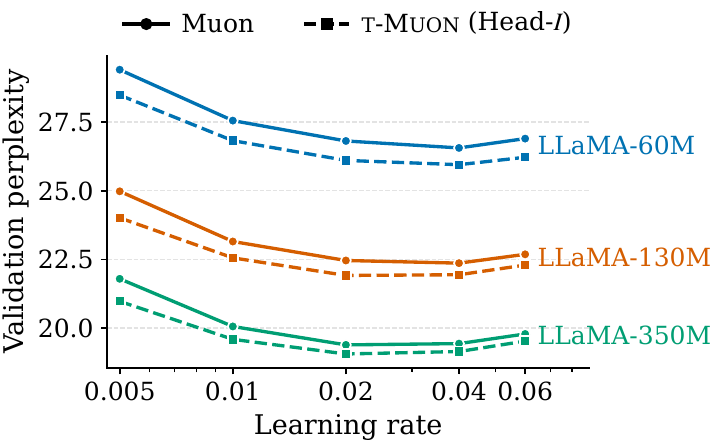}
        \caption{LLaMA-style models}
        \label{fig:lr-sweep-llama}
    \end{subfigure}
    \caption{
        Final validation perplexity as a function of the learning rate over the grids in \Cref{tab:lr-grids}.
    }
    \label{fig:lr-sweep}
\end{figure*}

\begin{figure*}[t]
    \centering
    \begin{subfigure}[t]{0.48\textwidth}
        \centering
        \includegraphics[width=\linewidth]
        {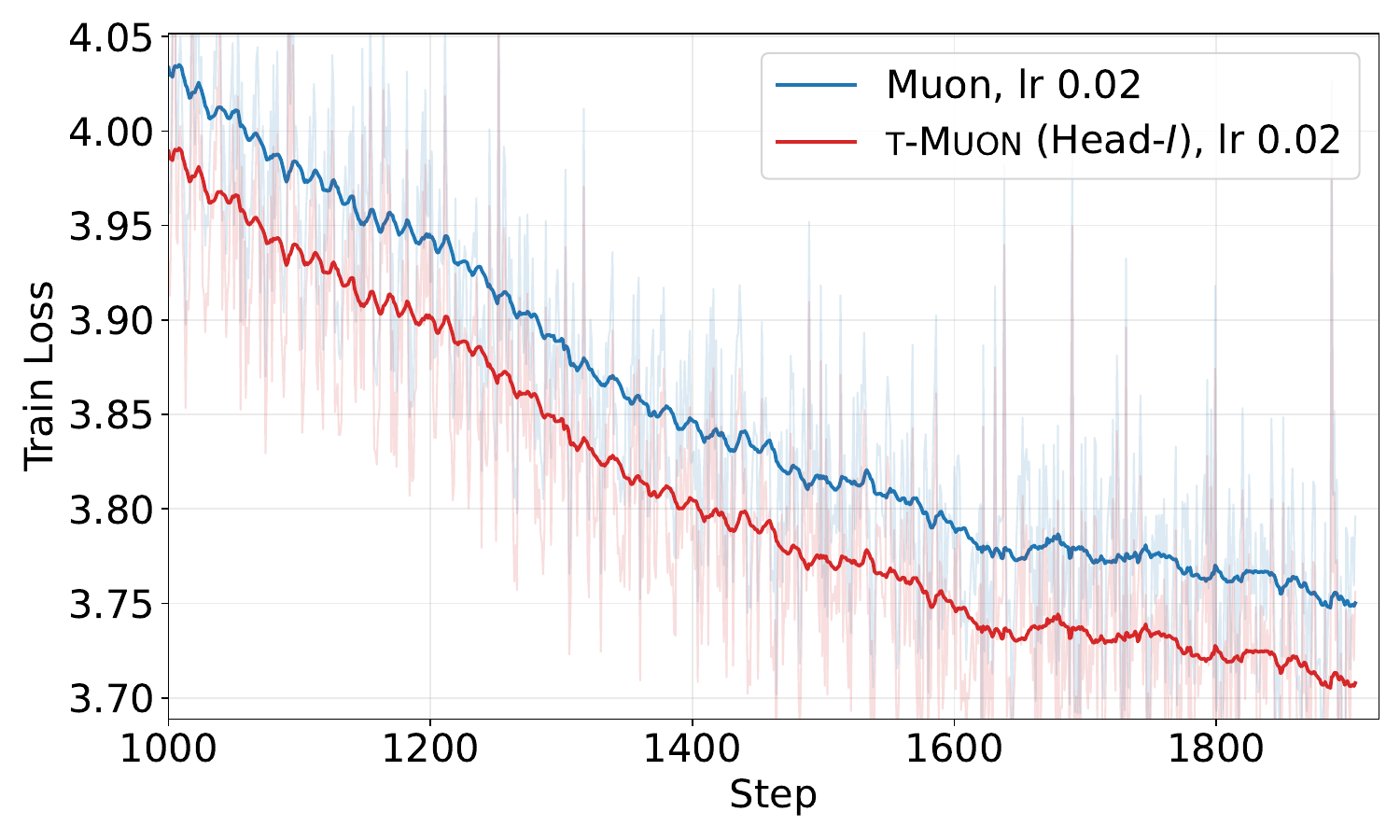}
        \caption{GPT-2-50M}
        \label{fig:gpt-50m-train}
    \end{subfigure}
    \hfill
    \begin{subfigure}[t]{0.48\textwidth}
        \centering
        \includegraphics[width=\linewidth]
        {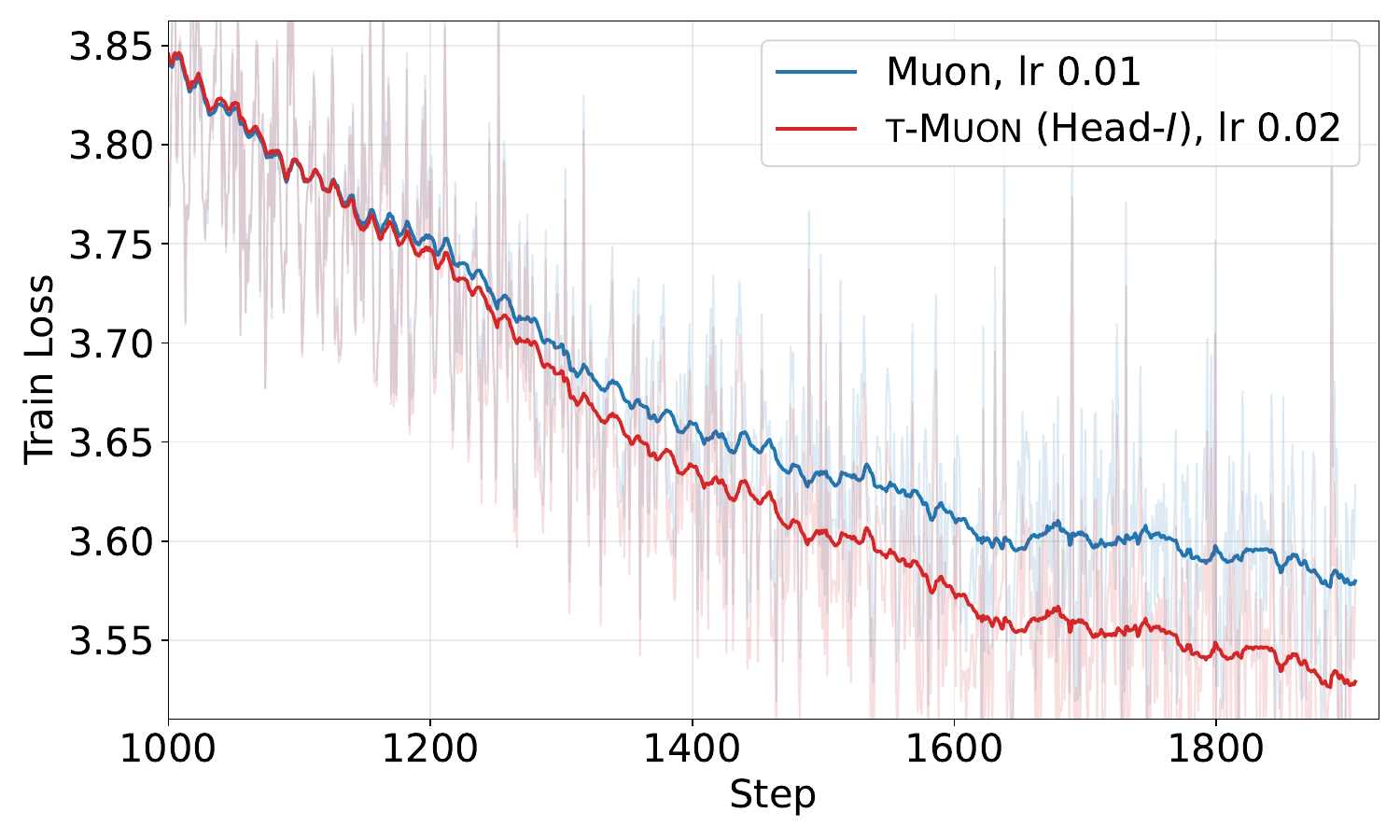}
        \caption{GPT-2-Small}
        \label{fig:gpt-small-train}
    \end{subfigure}

    \medskip

    \begin{subfigure}[t]{0.48\textwidth}
        \centering
        \includegraphics[width=\linewidth]
        {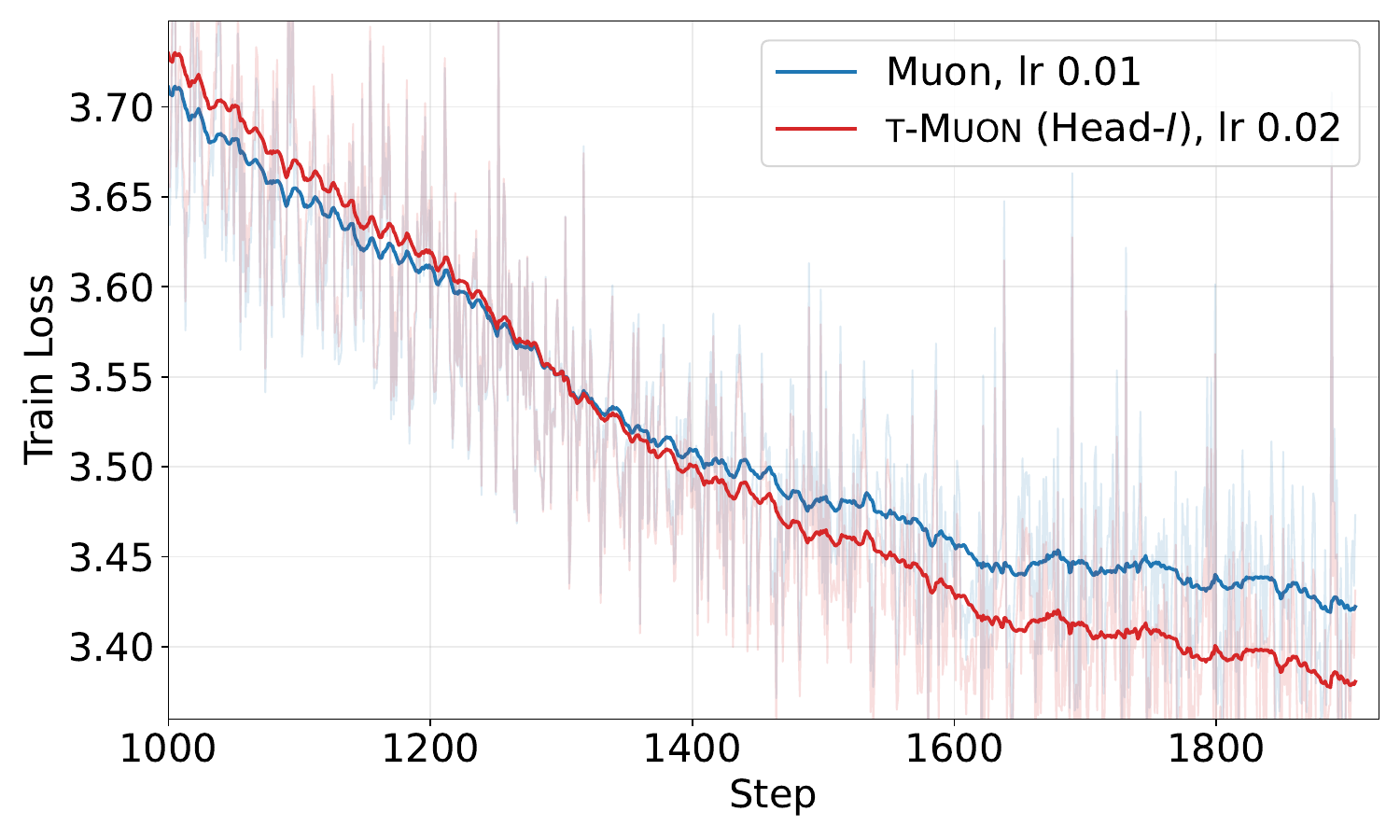}
        \caption{GPT-2-Base}
        \label{fig:gpt-base-train}
    \end{subfigure}
    \hfill
    \begin{subfigure}[t]{0.48\textwidth}
        \centering
        \includegraphics[width=\linewidth]
        {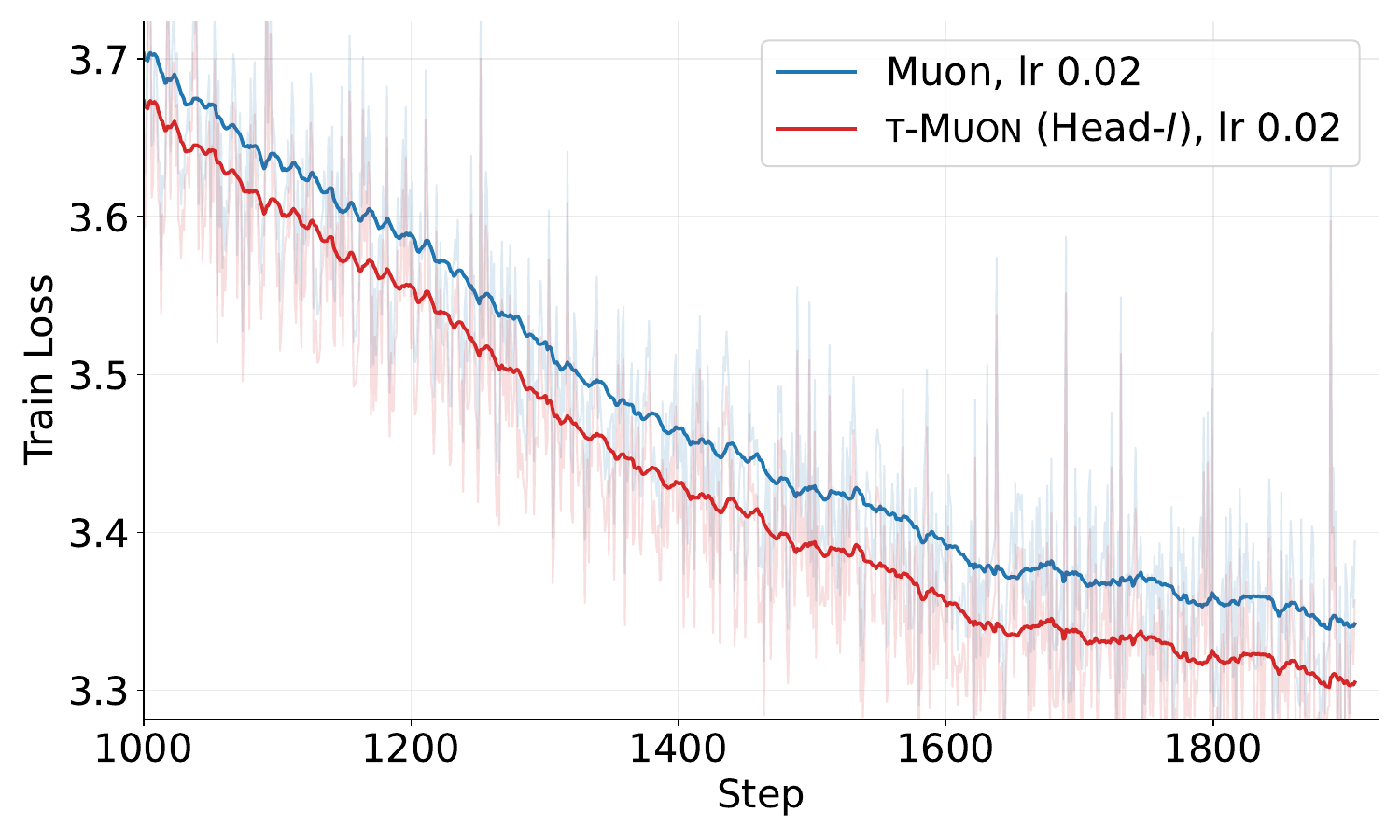}
        \caption{GPT-2-Large}
        \label{fig:gpt-large-train}
    \end{subfigure}

    \caption{
        Training-loss curves for GPT-2-style models trained on one billion FineWeb tokens.
    }
    \label{fig:gpt-train}
\end{figure*}

\begin{figure*}[p]
    \centering

    \begin{subfigure}[b]{0.65\textwidth}
        \centering
        \includegraphics[width=\linewidth]{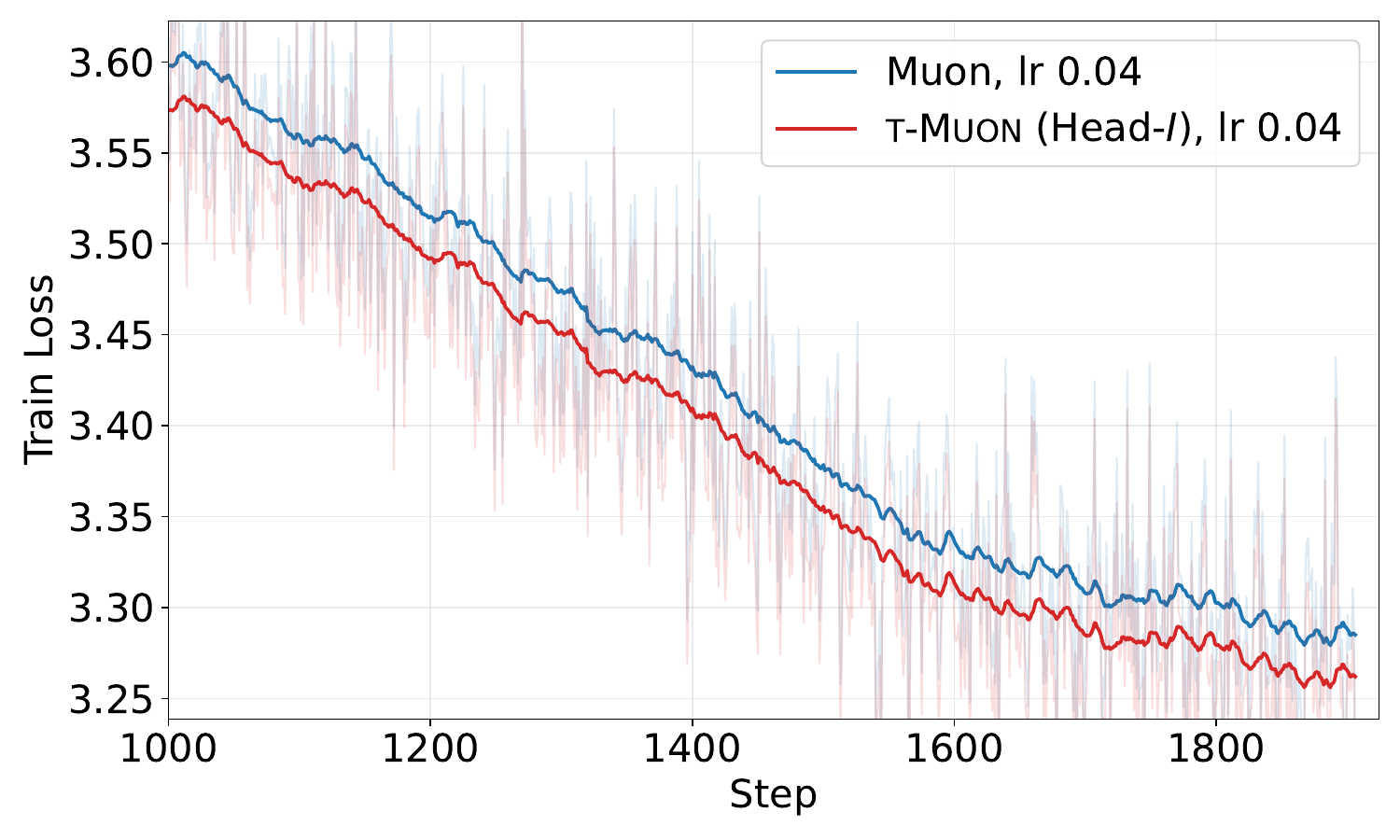}
        \caption{LLaMA-60M}
        \label{fig:llama-60m-train}
    \end{subfigure}

    \par\medskip

    \begin{subfigure}[b]{0.65\textwidth}
        \centering
        \includegraphics[width=\linewidth]{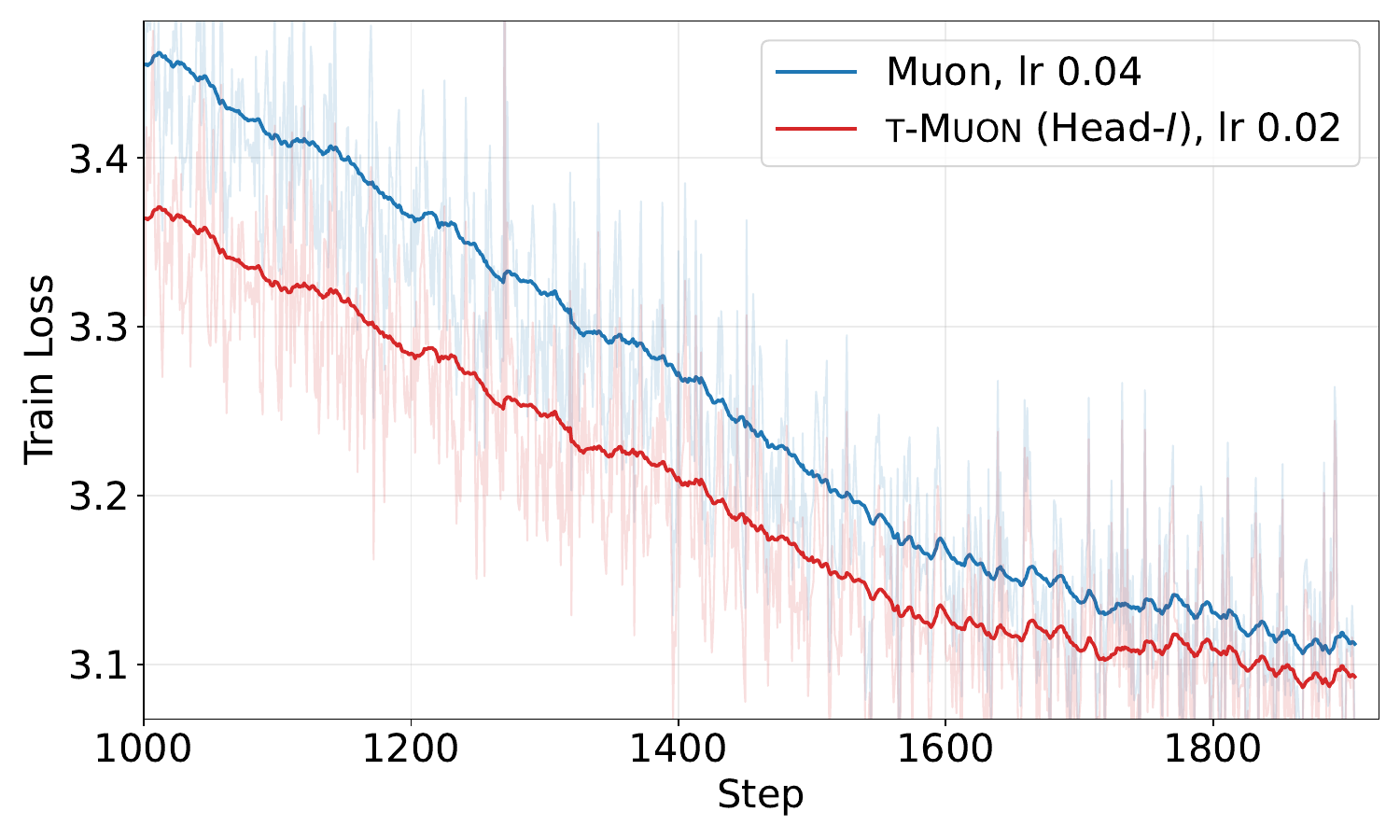}
        \caption{LLaMA-130M}
        \label{fig:llama-130m-train}
    \end{subfigure}

    \par\medskip

    \begin{subfigure}[b]{0.65\textwidth}
        \centering
        \includegraphics[width=\linewidth]{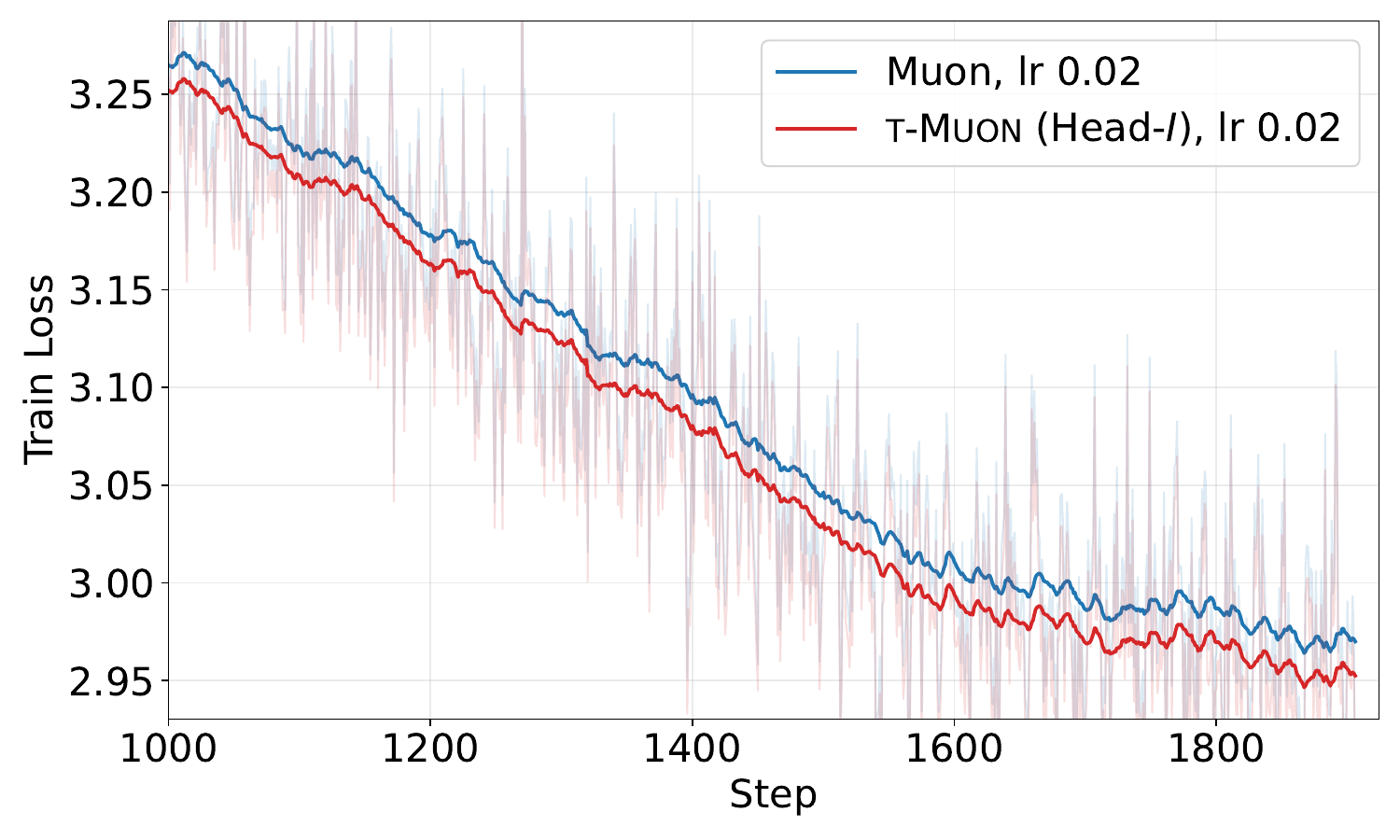}
        \caption{LLaMA-350M}
        \label{fig:llama-350m-train}
    \end{subfigure}

    \caption{Training-loss curves for LLaMA-style models trained
    on one billion FineWeb tokens.}
    \label{fig:llama-train}
\end{figure*}




\end{document}